\documentclass[twoside]{article}

\usepackage[accepted]{aistats2021}

\usepackage{tikz} 
\usepackage{amsfonts}
\usepackage{amsmath, amssymb, amsthm}
\usepackage{hyperref}
\usepackage{algpseudocode}
\usepackage{algorithm}
\usepackage{mathtools, bbm}
\usepackage{color}
\usepackage{caption}
\usepackage{subcaption}
\usepackage{comment}
\usepackage{natbib}

\newcommand{\del}[1]{}
\DeclareMathOperator*{\argmin}{arg\,min}

\newtheorem{lemma}{Lemma}
\newtheorem{theorem}{Theorem}

\newtheorem{corollary}{Corollary}

\newtheorem{assumption}{Assumption}

\newtheorem{definition}{Definition}

\def\X{\mathcal{X}}

\def\E{\mathbb{E}}
\def\V{\mathbb{V}}

\def\1{\mathbf{1}}
\def\P{\mathbb{P}}
\def\R{\mathbb{R}}

\def\E{\mathbb{E}}
\def\P{\mathbb{P}}
\def\R{\mathbb{R}}
\def\V{\mathcal{V}}
\def\X{\mathcal{X}}
\def\argmin{\text{argmin}}
\def\t{\top}
\def\1{\bm{1}}
\def\HighProbConst{c_0}
\def\ours{RIPS}
\def\tr{\mathrm{Trace}}

\newcommand{\mc}[1]{\mathcal{#1}}
\newcommand{\mb}[1]{\mathbf{#1}}
\def\t{\top}

\newcommand{\norm}[1]{\left\lVert#1\right\rVert}

\begin{document}

\twocolumn[

\aistatstitle{High-Dimensional Experimental Design and Kernel Bandits}

\aistatsauthor{ Romain Camilleri \And Julian Katz-Samuels \And  Kevin Jamieson }

\aistatsaddress{ University of Washington \\ \texttt{camilr@cs.washington.edu} \And University of Wisconsin \\ \texttt{katzsamuels@wisc.edu} \And University of Washington \\ \texttt{jamieson@cs.washington.edu} } ]

\begin{abstract}
In recent years methods from optimal linear experimental design have been leveraged to obtain state of the art results for linear bandits. 
A design returned from an objective such as $G$-optimal design is actually a probability distribution over a pool of potential measurement vectors. 
Consequently, one nuisance of the approach is the task of converting this continuous probability distribution into a discrete assignment of $N$ measurements. 
While sophisticated rounding techniques have been proposed, in $d$ dimensions they require $N$ to be at least $d$, $d \log(\log(d))$, or $d^2$ based on the sub-optimality of the solution.
In this paper we are interested in settings where $N$ may be much less than $d$, such as in experimental design in an RKHS where $d$ may be effectively infinite.  
In this work, we propose a rounding procedure that frees $N$ of any dependence on the dimension $d$, while achieving nearly the same performance guarantees of existing rounding procedures.
We evaluate the procedure against a baseline that projects the problem to a lower dimensional space and performs rounding which requires $N$ to just be at least a notion of the effective dimension. We also leverage our new approach in a new algorithm for kernelized bandits to obtain state of the art results for regret minimization and pure exploration. 
An advantage of our approach over existing UCB-like approaches is that our kernel bandit algorithms are also robust to model misspecification. 
\end{abstract}

\section{Introduction}\label{sct:setting}
This work studies a non-parametric multi-armed bandit game through the lens of experimental design. 
Fix a finite set of measurements $\mc{X} \subset \R^d$ and a function $\mu: \mc{X} \rightarrow \R$. 
We consider the following game between a learner and nature: 
at each time $t=1 \ldots T$, the learner requests $x_t \in \mc{X}$ and nature immediately reveals 
\begin{align*}
    y_t = \mu_{x_t}+\xi_t 
\end{align*}
where $\{\xi_t\}_{t=1}^T$ is a sequence of independent, mean-zero random variables with bounded variance.
We are interested in two objectives:

\paragraph{Regret minimization}
In this setting, we evaluate the performance of an algorithm choosing actions $\{x_t\}_{t=1}^T$ by its cumulative regret: $R_T = \max_{x \in \mc{X}}\sum_{t=1}^T \left(\mu_x - \mu_{x_t}\right)$.

\paragraph{Pure exploration in the PAC setting}
For a tolerance $\epsilon \geq 0$ and confidence level $\delta\in (0, 1)$, the aim of the learner in pure exploration is to sequentially take samples until a learner-defined stopping criterion is met, at which time the learner outputs an arm $\widehat{x} \in \mc{X}$ such that $\mu_{\widehat{x}} \geq \max_{x \in \mc{X}} \mu_{x} - \epsilon$ with probability at least $1 - \delta$. 

To aid us in our objectives, we assume some structure on the reward function $\mu$.

\begin{assumption}
There exists a known feature map $\phi : \R^d \mapsto \mc{H}$ that maps each $x \in \mc{X}$ to a (possibly infinite dimensional) Hilbert space $\mc{H}$, and moreover, there exists a $\theta_* \in \mc{H}$ and $h \geq 0$ such that $\max_{x\in\mc{X}} |\mu_x - \langle \theta_*,\phi(x)\rangle_{\mc{H}}| \leq h$.
\end{assumption}
Consequently, if $h$ is not too big, the expected value of each of the observations $y_t$ is nearly a linear function of its associated features $\phi(x_t)$. We say the model is \emph{misspecified} when $h >0$, and otherwise the setting is well-specified and reduces to the classical stochastic setting when $h=0$.

\begin{assumption}
Rewards are bounded $\max_{x\in\mc{X}}|\mu_x| \leq B$.
\end{assumption}

\begin{assumption}
For every time $t$, the additive stochastic noise $\xi_t$ is independent, mean-zero with $\E[\xi_t^2] \leq \sigma^2$.
\end{assumption}

While we assume the learner knows $B$ and $\sigma^2$, we assume that the learner \emph{does not} know the extent of the model misspecification $h \geq 0$. 
Note that we do \emph{not} assume $\xi_t$ is bounded, indeed, it can even be heavy tailed.

\subsection{Elimination algorithms and experimental design}\label{sec:elim_algorithm_motivation}
Whether the model is misspecified ($h > 0$) or not ($h=0$), a popular class of algorithms for both the objectives of regret minimization and pure exploration is known as \emph{elimination algorithms}. 
Elimination algorithms proceed in stages, maintaining a set $\widehat{\mc{X}} \subset \mc{X}$ of candidates that may achieve $\max_{x \in \mc{X}} \mu_x$ given all previous observations. At the beginning of the stage $\ell \geq 1$ the algorithm decides which measurements to take, nature reveals the observations, and the stage ends by constructing an estimate $\widehat{\mu}_{(\cdot)}$ of $\mu_{(\cdot)}$ and removing all elements $x \in \widehat{\mc{X}}$ from $\widehat{\mc{X}}$ where $\max_{x' \in \widehat{\mc{X}}} \widehat{\mu}_{x'} - \widehat{\mu}_x > \epsilon_\ell$. 
This process is repeated indefinitely in the case of regret minimization, or until $\widehat{\mc{X}}$ contains a single element in the case of pure exploration.
To be as effective as possible at discarding as many candidates as possible in the elimination stage (without discarding the best arm), a natural strategy of selecting how many and which measurements to take in the beginning of the round is to select $x_1,\dots,x_n \in \mc{X}$ to accurately estimate the differences of the estimates
\begin{align}
    \max_{x,x' \in \widehat{\mc{X}}} (\widehat{\mu}_{x'} - \widehat{\mu}_x) - (\mu_{x'} - \mu_x)  \leq \epsilon_\ell. \label{eqn:planning_objective}
\end{align}
If $x_* := \arg\max_{x \in \mc{X}} \mu_x$ and $x_* \in \widehat{\mc{X}}$ at the start of the round, then we have that $x_*$ will not be eliminated at the end since 
\begin{align*}
    \max_{x' \in \widehat{\mc{X}}} \widehat{\mu}_{x'} - \widehat{\mu}_{x_*} &\leq \max_{x' \in \widehat{\mc{X}}} \mu_{x'} - \mu_{x_*} + \epsilon_\ell \leq \epsilon_\ell.
\end{align*}
And moreover, it is straightforward to show that after the discarding step of stage $\ell$, $\max_{x \in \widehat{\mc{X}}}{\mu}_{x_*} - \mu_x \leq 2 \epsilon_\ell$.
To guide our choice of $x_1,\dots,x_n \in \mc{X}$ to achieve \eqref{eqn:planning_objective}, we exploit the assumed (nearly) linear model of above.

\subsection{Optimal experimental design and the problem of rounding continuous designs}\label{sct:rounding}

This section introduces the method of experimental design with the goal of achieving \eqref{eqn:planning_objective} by taking as few total samples as possible. 
Shortly, we will consider the case when $h > 0$ and $\phi$ is an arbitrary feature map.
But for now, let us make the simplifying assumption that $h=0$, $\phi$ is the identity map so that  $\mu_x = \langle \theta_* , x \rangle$, and $\xi_t \sim \mc{N}(0,\sigma^2)$.
Thus, if at time $t$ we select $x_t \in \mc{X} \subset \R^d$ we observe $\langle \theta_*, x_t \rangle + \xi_t$.
Suppose we observed pairs $\{(x_t,y_t)\}_{t=1}^T$ where each $x_t \in \mc{X}$ was chosen independently of any $y_s$ for $s \leq t$.
If we wished to achieve \eqref{eqn:planning_objective} for $\widehat{\mc{X}} \subset \mc{X}$ with
$\mu_x = \langle x, \theta_* \rangle$, perhaps the most natural way forward would be to compute the least squares estimator 
$ \widehat{\theta}_{LS} = \arg\min_\theta \sum_{t=1}^T (y_t - \langle x_t,\theta \rangle)^2$, and set $\widehat{\mu}_x = \langle \widehat{\theta}_{LS},x \rangle$.
Then \eqref{eqn:planning_objective} is equivalent to $\max_{v \in \mc{V}} \langle \widehat{\theta}_{LS} - \theta_* , v \rangle \leq \epsilon_\ell$ with $\mc{V} = \widehat{\mc{X}} - \widehat{\mc{X}}$.
By a standard sub-Gaussian tail-bound \cite{lattimore2020bandit}, we have with probability at least $1-\delta$ that for all $v \in \mc{V}\subset\R^d$
\begin{align}
    |\langle v,\! \widehat{\theta}_{LS} \!-\! \theta_* \rangle| \! \leq \! \| v \|_{( \sum_{t=1}^T \!x_t x_t^\top )^{-1}} \! \sqrt{ 2 \sigma^2\! \log(2|\V|/\delta) },\label{eq:least_squares_bound}
\end{align}  
where we adopt the notation $\|z\|_A = \sqrt{z^\top A z}$ for any $z \in \R^d$ and symmetric semi-definite positive $A$.
Note that this error bound only depends on those $x_t$ measurements that we choose \emph{before} any responses $y_t$ are observed. This allows us to plan, that is, choose the $T$ measurement vectors to minimize the RHS of \eqref{eq:least_squares_bound}. Unfortunately, this minimization problem is known to be NP-hard \cite{pukelsheim2006optimal,allen2017near}. 
As a consequence, approximation algorithms based on the relaxation 
\begin{align}
    \bar{\lambda} = \argmin_{\lambda \in \triangle_{\mc{X}}} \max_{v \in \mc{V}} v^\top \left( \textstyle\sum_{x \in \mc{X}} \lambda_x x x^\top \right)^{-1} v \label{eq:continuous_design_objective}
\end{align}
have been proposed. These first solve for $\bar{\lambda}$ and ``round'' this to a discrete allocation of measurements.
 
\paragraph{Deterministic rounding} Perhaps the simplest scheme is to obtain a solution $\bar{\lambda}$ of \eqref{eq:continuous_design_objective} and then sample $x \in \mc{X}$ exactly $\lceil \bar{\lambda}_x T \rceil$ times. In the worst case, this will result in $|\text{support}(\bar{\lambda})|$ additional measurements than the intended $T$.
Caratheodory's theorem provides a polynomial-time algorithm for constructing $\tilde{\lambda} \in \triangle_{\mc{X}}$ such that $\sum_{x \in \mc{X}} \tilde{\lambda}_x x x^\top = \sum_{x \in \mc{X}} \bar{\lambda}_x x x^\top$ and $|\text{support}(\tilde{\lambda})| \leq (d+1)d/2$. However, more sophisticated rounding procedures exist. \cite{allen2017near} inflates the RHS of \eqref{eq:least_squares_bound} by a constant factor while only requiring that $T = \Omega(d)$. When $\V = \X$, another strategy is to solve the optimization problem \eqref{eq:continuous_design_objective} with a Frank-Wolfe style algorithm that is terminated only after $O(d \log\log(d))$ iterations so that the rounding according to the naive ceiling operation only inflates $T$ by the number of iterations which is $O(d \log\log(d))$ \cite{todd2016minimum}. 

\paragraph{Stochastic rounding} 
Another basic rounding algorithm simply samples $x_1,\ldots, x_T \sim \bar{\lambda}$. 
Unfortunately, using the least squares estimator $\widehat{\theta}_{LS}$, we may have that $\sum_{t=1}^T x_t x_t^\top $ deviates dramatically from $T\sum_{x \in \mc{X}} \bar{\lambda}_x x x^\top$ for moderate $T$, thus any guarantees require $T$ to be $\text{poly}(d)$ and moreover, performance relies on the spectrum of $\sum_{x \in \mc{X}} \bar{\lambda}_x x x^\top$ \cite{rizk2020refined}.
As a consequence, \cite{tao2018best} proposed using the inverse propensity score (IPS) estimator $\widehat{\theta}_{IPS} := ( \sum_{x \in \mc{X}} \bar{\lambda}_x x x^\top )^{-1} ( \frac{1}{T} \sum_{t=1}^T x_t y_t )$. 
From \cite{tao2018best}, with probability at least $1-\delta$ we have for all $v \in \V$ simultaneously 
\begin{align}
&|\langle v, \widehat{\theta}_{IPS} - \theta_* \rangle| \leq \sqrt{\frac{2\sigma^2 \norm{v}_{A(\bar\lambda)^{-1}}^2 \log(2|\mc{V}|/\delta)}{T}} \label{eqn:ips_estimator_bound}\\
&\quad+ \frac{\log(2|\mc{V}|/\delta)(1 + \max_{x \in \X} |v^\t A(\bar\lambda)^{-1} x|)}{T}. \nonumber
\end{align}
where $A(\lambda):= \sum_{x \in \mc{X}} \lambda_x x x^\top$. The second term of \eqref{eqn:ips_estimator_bound} accounts for potentially rare but large deviations of size $\max_{x \in \X} |v^\t A(\bar\lambda)^{-1} x|$. 
Sadly, this second term is cumbersome in analyses since it can dominate the first term, and it cannot be removed in the worst-case. 
A final class of algorithms rely on \emph{proportional volume sampling}, or sampling from a determinantal point process (DPP), but are limited to specific optimality criteria \cite{nikolov2019proportional,derezinski2020bayesian}.

\subsection{Main contributions}
The main contributions of this paper include a novel scheme for experimental design and its application to kernel bandits. 
\begin{itemize}
    \item We propose an estimator $\widehat{\theta}_{\ours}$ that overcomes many of the shortcomings of the prior art reviewed in Section~\ref{sct:rounding} for $h=0$ and $\phi\equiv$ identity. For any fixed $\theta_* \in \R^d$, $\mc{V} \subset\R^d, \mc{X} \subset \R^d$, $\lambda \in \triangle_{\mc{X}}$, and $T \in \mathbb{N}$, if $T$ samples are drawn randomly according to $\lambda$ to construct $\widehat{\theta}_{\ours}$, then with probability at least $1-\delta$ we have for all $v \in \mc{V}$
    \begin{align*}
&|\langle v, \widehat{\theta}_{\ours} - \theta_* \rangle| \\
&\leq \| v \|_{( \sum_{x \in \mc{X}} \lambda_x x x^\top)^{-1}}  \sqrt{ \frac{c (\sigma^2+B^2) \log(2|\V|/\delta)}{T} }
\end{align*}
for an absolute constant $c$. 
Note that our method puts no restrictions on $T$ but matches the ideal discrete allocation of \eqref{eq:least_squares_bound} up to a constant by realizing that 
\begin{align*}
    \frac{\inf_{\lambda \in \triangle_{\mc{X}}} \max_{v \in \mc{V}}\| v \|_{( \sum_{x \in \mc{X}} T \lambda_x x x^\top)^{-1}}}{\min_{\{x_t\}_{t=1}^T \in \mc{X}} \max_{v \in \mc{V}}\| v \|_{( \sum_{t=1}^T x_t x_t^\top )^{-1}}} \leq 1.
\end{align*}
We also note that we only assume the stochastic noise has bounded variance and do not rule out heavy-tailed distributions. 
The estimator $\widehat{\theta}_{\ours}$ is a special case of our more general estimator.
 
\item We extend our estimator to the misspecified setting where $h\geq 0$ and to use feature maps $\phi : \R^d \rightarrow \mc{H}$ for an RKHS $\mc{H}$.
When $\mc{H}$ can represent a high or even an infinite dimensional space, restrictions on $T$ based on the dimension start to become paramount.  
For any fixed $\theta_* \in \mc{H}$, $\mc{V} \subset\mc{H}, \mc{X} \subset \R^d$, $\lambda \in \triangle_{\mc{X}}$, $T \in \mathbb{N}$, and $\gamma \geq 0$, if $T$ samples are drawn randomly according to $\lambda$ to construct $\widehat{\theta}_{\ours}(\gamma)$, then with probability at least $1-\delta$ we have for all $v \in \mc{V}$
\begin{align*}
|\langle  v, & \widehat{\theta}_{\ours}(\gamma) - \theta_* \rangle| \leq  \| v \|_{( \sum_{x \in \mc{X}} \lambda_x \phi(x) \phi(x)^\top\! + \gamma I )^{-1}}  \\ &\times\Big(\sqrt{\gamma} \|\theta_*\|_2 + h \!+\! \! \sqrt{ \!\frac{c (\sigma^2\! +\! B^2)\! \log(2|\V|/\delta)}{T} } \Big).
\end{align*}
Note that since $\mc{H}$ may be infinite-dimensional, the estimator $\widehat{\theta}_{\ours}(\gamma)$ is constructed implicitly and is implemented through kernel evaluations only.
\item We empirically compare $\widehat{\theta}_{\ours}(\gamma)$ to the sampling and estimator pairs of Section~\ref{sct:rounding} and show that $\widehat{\theta}_{\ours}(\gamma)$ is competitive on both finite dimensional $G$-optimal design as well as its regularized RKHS variant sometimes called Bayesian experimental design. 

\item We employ $\widehat{\theta}_{\ours}(\gamma)$ in a novel elimination style algorithm for kernel bandits. Our regret bounds match state of the art results in the well-specified setting, and are the first linear bounds that we are aware of for the misspecified setting. 
In addition, we state an instance-dependent pure-exploration result for identifying an $\epsilon$-good arm with probability at least $1-\delta$ that compares favorably to known lower bounds. 
One advantage of our algorithm over prior kernel bandits and Bayesian Optimization algorithms \cite{srinivas2009gaussian,valko2013finitetime,frazier2018tutorial} is that our approach naturally allows for taking batches of pulls per round. 

\end{itemize}

\section{Robust Inverse Propensity Score (\ours) estimator}
In this section we introduce the $\widehat{\theta}_{\ours}$ estimator. 
In finite dimensions, our estimator first constructs $\widehat{\theta}_{IPS}$ but then to avoid the large deviations term of \eqref{eqn:ips_estimator_bound} applies robust mean estimation on each $\langle v, \theta_* \rangle$ to obtain a $\widehat{\theta}_{\ours}$ which is consistent with all of these estimates. 
When we move to an RKHS setting, we add regularization to avoid vacuous bounds and account for the introduced bias. 
The bias of misspecification is handled similarly.  
We begin with robust mean estimation.
\begin{definition}
 Let $X_1,\ldots, X_n$ be i.i.d. random variables with mean $\bar{x}$ and variance $\nu^2$. Let $\delta \in (0,1)$. We say that $\widehat{\mu}(X_1,\ldots,X_n)$ is a \emph{$\delta$-robust estimator} if there exist universal constants $c_1,\HighProbConst > 0$ such that if $n \geq c_1 \log(1/\delta)$, then with probability at least $1-\delta$
\begin{align*}
|\widehat{\mu}( \{X_t\}_{t=1}^n )- \bar{x}| \leq \HighProbConst \sqrt{\frac{\nu^2 \log(1/\delta)}{n}}.
\end{align*}
\end{definition}
Examples of $\delta$-robust estimators include the median-of-means estimator and Catoni's estimator \cite{lugosi2019mean}. 
This work employs the use of the Catoni estimator which satisfies $|\widehat{\mu}( \{X_t\}_{t=1}^n )- \bar{x}| \leq \sqrt{ \frac{2 \nu^2 \log(1/\delta)}{n - 2 \log(1/\delta)}}$ for $n > 2 \log(1/\delta)$ which leads to an optimal leading constant as $n \rightarrow \infty$.
We will use a separate robust mean estimate for each $v \in \mc{V}$. In particular, to estimate $\langle v, \theta_* \rangle$ we use $\widehat{\mu}( \{ v^\t A^{(\gamma)}(\lambda)^{-1} \phi(x_t) y_t \}_{t=1}^T )$ where
\begin{align}
    A^{(\gamma)}(\lambda) := \sum_{x \in \mc{X}} \lambda_x \phi(x) \phi(x)^\top + \gamma I.\label{eq:A_lambda}
\end{align}

\begin{figure}[h]
\begin{algorithm}[H] 
  \caption{RIPS for Experimental Designs in an RKHS}
\begin{algorithmic}
\State {\bfseries Input:} Finite sets $\mc{X} \subset \R^d$ and $\mc{V} \subset \mc{H}$, feature map $\phi: \R^d \rightarrow \mc{H}$, number of samples $\tau$, regularization $\gamma>0$, robust mean estimator $\widehat{\mu}: \R^* \rightarrow \R$
\begin{align}\label{eqn:RKHS_exp_design_in_RIPS_alg}
    \lambda^* := \arg\min_{\lambda \in \triangle_{\mc{X}}} \max_{v \in \mc{V}}\|v\|_{\left(\sum_{x}\lambda_x \phi(x)\phi(x)^\top + \gamma I\right)^{-1}}
\end{align}
\State Randomly draw $\widetilde{x}_1, \ldots, \widetilde{x}_\tau$ from $\mc{X}$ according to $\lambda^*$
\State Set $W^{(v)} = \widehat{\mu}( \{ v^\t A^{(\gamma)}(\lambda^*)^{-1} \phi(\widetilde{x}_t) \widetilde{y}_t \}_{t=1}^\tau )$
\State Set $\displaystyle\widehat{\theta} := \arg\min_{\theta} \max_{v \in \mc{V}} \frac{ | \langle \theta, v \rangle - W^{(v)}| }{\|v\|_{(\sum_{x \in \mc{X}} \lambda_x^* \phi(x) \phi(x)^\top + \gamma I)^{-1}}}$
\State {\bfseries Return:} $\{W^{(v)}\}_{v \in \mc{V}}$, $\widehat{\theta}$
\end{algorithmic}
\end{algorithm}
\caption{In this work, we assume each element in $\mc{V}$ is a linear combination of $\phi \circ \mc{X}$ which makes all quantities well-defined and can be computed using kernel evaluations $k(x,x') := \langle\phi(x),\phi(x')\rangle_{\mc{H}}$. Moreover, Equation~\ref{eqn:RKHS_exp_design_in_RIPS_alg} is convex with gradients that can be computed using kernel evaluations (See Section~\ref{sec:computing_with_kernel_evaluations}).}
\label{alg:RIPS_exp_design_for_RKHS}
\end{figure}

Our {\ours} procedure for experimental design in an RKHS is presented in Figure~\ref{alg:RIPS_exp_design_for_RKHS}. It has the following guarantee.
\begin{theorem}\label{thm:robust_estimator_lemma}
Fix any finite sets $\mc{X} \subset \R^d$ and $\mc{V} \subset \mc{H}$, feature map $\phi: \R^d \rightarrow \mc{H}$, number of samples $\tau$ and regularization $\gamma>0$. If the {\ours} procedure of Figure~\ref{alg:RIPS_exp_design_for_RKHS} is run with $\tfrac{\delta}{|\mc{V}|}$-robust mean estimator  $\widehat{\mu}(\cdot)$ and if $\tau \geq c_1 \log(|\mc{V}|/\delta)$ then with probability at least $1-\delta$, we have 
\begin{align*}
    \max_{v \in \mc{V}} &\frac{ | W^{(v)} - \langle \theta_*, v \rangle| }{\|v\|_{(\sum_{x \in \mc{X}} \lambda_x \phi(x) \phi(x)^\top + \gamma I)^{-1}} } \\
    &\leq  \sqrt{\gamma} \|\theta_* \|_2 + h +   \HighProbConst \sqrt{\tfrac{(B^2 + \sigma^2)}{\tau}\log(2|\mc{V}|/\delta) } ,
\end{align*}
Moreover, $W^{(v)} = \widehat{\mu}( \{ v^\t A^{(\gamma)}(\lambda)^{-1} \phi(x_t) y_t \}_{t=1}^\tau )$ can be replaced by $\langle \widehat{\theta},v\rangle$ by multiplying the RHS by a factor of $2$.
\end{theorem}
\begin{proof}[Proof sketch]
Due to the regularization and potential misspecification if $h >0$, each $v^\t A^{(\gamma)}(\lambda)^{-1} \phi(x_t) y_t$ is biased.
Thus, we apply the guarantee of $W^{(v)} = \widehat{\mu}( \{ v^\t A^{(\gamma)}(\lambda)^{-1} \phi(x_t) y_t \}_{t=1}^\tau )$ to the expectation of its arguments.
The triangle inequality followed by repeated applications of Cauchy-Schwartz yields
\begin{align*}
    |W^{(v)} - \langle v, \theta_* \rangle| 
    \leq& |W^{(v)} - \E[ v^\t A^{(\gamma)}(\lambda)^{-1} \phi(x_1) y_1 ]| \\
     &+|\E[ v^\t A^{(\gamma)}(\lambda)^{-1} \phi(x_1) y_1 ] - \langle v, \theta_* \rangle| \\
     \leq&\HighProbConst \sqrt{\frac{\nu^2 \log(1/\delta)}{\tau}} +\sqrt{\gamma} \|\theta_* \|_2 + h
\end{align*}
where we obtain an upper bound on the variance $\nu^2$ by
\begin{align*}
&\mathbb{V}\text{ar}( v^\t\! A^{(\gamma)}(\lambda)^{-1}\! \phi(x_1) y_1 ) \!\leq\! \E[ ( v^\t \!A^{(\gamma)}(\lambda)^{-1} \phi(x_1) y_1 )^2 ] \\
&\quad= \E\left[ \left( v^\t\! A\!^{(\gamma)}\!(\lambda)^{-1}\! \phi(x_1) \right)^2\mu_{x_1} ^2 \right] \\
&+\! \E\left[\! \left( v^\t\! A^{(\gamma)}\!(\lambda)^{-1} \phi(x_1)  \right)^2 \!\xi_1^2 \right]\!\leq\! (B^2 \!+\! \sigma^2) \|v\|_{A^{(\gamma)}\!(\lambda)^{-1}}^2.
\end{align*}
\end{proof}
\subsection{Practical implementation of the algorithms}
The construction of $\widehat{\theta}$ in the algorithms may--at first glance--look confusing in the infinite dimensional case. In actuality, the equivalent dual representation $\widehat{\theta} = \sum_{i=1}^{|\mc{X}|}\alpha_i \phi(x_i)$ would be used. That is, the potentially infinite dimensional object $\widehat{\theta}$ is represented by a finite dimensional weight vector $\alpha \in \R^{|\mc{X}|}$.
With that, the optimizations in the algorithms (e.g., to compute the {\ours} estimator) are over the dual vector $\alpha \in \R^{|\mc{X}|}$, and inner products $\langle \widehat{\theta}, v \rangle = \sum_{i=1}^{|\mc{X}|}\alpha_i \langle  \phi(x_i), v \rangle$ are computed using the kernel matrix of $\mc{X}$ since in all instances of $v$ used in the algorithms, $v$ is a linear combination of $\{\phi(x)\}_{x \in \mc{X}}$.

\subsection{Comparison to IPS estimator}
Note the difference between the bound of {\ours} in Theorem~\ref{thm:robust_estimator_lemma} with the bound of the IPS estimator stated in equation~\eqref{eqn:ips_estimator_bound}. Consider the setting of equation~\eqref{eqn:ips_estimator_bound}. Ignoring log factors and constants, the confidence bound of the IPS estimator essentially scales as $\sqrt{\frac{\sigma^2 \norm{v}_{A(\bar\lambda)^{-1}}^2 }{T}} + \frac{ \max_{x \in \X} |v^\t A(\bar\lambda)^{-1} x|}{T}$, while the confidence bound of RIPS essentially scales as $\sqrt{\frac{\sigma^2 \norm{v}_{A(\bar\lambda)^{-1}}^2 }{T}}$. 
It can be shown that in the instance in the experiment corresponding to figure~\ref{fig:ips}, the term $ \frac{ \max_{x \in \X} |v^\t A(\bar\lambda)^{-1} x|}{T} \approx \frac{d}{T}$ while $\sqrt{\frac{\sigma^2 \norm{v}_{A(\bar\lambda)^{-1}}^2 }{T}} \approx \frac{\sqrt{d}}{\sqrt{T}}$. Thus, the first term dominates by a polynomial factor in the dimension until $T \geq d$, and the experiment shows that indeed the IPS estimator has larger deviations than RIPS, as suggested by the above upper bounds.

\subsection{Experimental Design optimization in an RKHS}\label{sec:computing_with_kernel_evaluations}
We now discuss how to actually compute an allocation in a potentially infinite dimensional RKHS $\mc{H}$. The following lemma will be helpful and is proved in the appendix.
\begin{lemma}\label{lmm:compute_kernel_with_alloc}
If $ A_\lambda \!= \!\sum_{x\in\mc{X}}\lambda_x \phi(x)\phi(x)^\top$ then for $a, b \in \mc{H}$
\begin{align*}
    a^\top \!(A_\lambda\! +\! \gamma I)^{-1} b 
    \!=\! \frac{1}{\gamma} a^\top b - \frac{1}{\gamma}k_\lambda(a)^\top (K_\lambda \!+\! \gamma I_{|\mc{X}|})^{-1} k_\lambda(b)
\end{align*}
with $k_\lambda(\cdot) \in \R^{|\mc{X}|}$ so that for any $c\in\mc{H}$, $[k_\lambda(c)]_i =  \sqrt{\lambda_i} \phi(x_i)^\top c$, and  $K_\lambda \in \R^{|\mc{X}| \times |\mc{X}|}$ so that
\begin{align*}
    \left[ K_\lambda \right]_{i, j} = \sqrt{\lambda_i} \sqrt{\lambda_j} \phi(x_j)^\top\phi(x_j) =: \sqrt{\lambda_i} \sqrt{\lambda_j}  k(x_i,y_i).
\end{align*}
For $x \in \mc{X}$, $[k_\lambda(x)]_i = \sqrt{\lambda_i} \phi(x_i)^\top \phi(x) = \sqrt{\lambda_i} k(x_i,x)$.
\end{lemma}
If we call $f(\lambda)$ the argument of Equation~\ref{eqn:RKHS_exp_design_in_RIPS_alg} in Figure~\ref{alg:RIPS_exp_design_for_RKHS}, and 
$\bar{v} \!\in\! \arg\max_{v\in \mc{V}} v^\top \left(\sum_{x\in\mc{X}}\lambda_x \phi(x)\phi(x)^\top + \gamma I\right)^{-1} v$ then the computation of the gradient of $\lambda \mapsto f(\lambda)$ equals
\begin{align*}
    &[\nabla_\lambda f(\lambda)]_i = - \big(\bar{v}^\top (\sum_{x\in\mc{X}}\lambda_x \phi(x)\phi(x)^\top + \gamma I)^{-1} \phi(x_i)\big)^2.
\end{align*}
Importantly, in this work $\mc{V}$ will always be a linear combination of $\{\phi(x)\}_{x \in \mc{X}}$ (e.g. $\mc{V} = \mc{X} - \mc{X}$), thus the last quantity can be computed only using kernel evaluations thanks to Lemma~\ref{lmm:compute_kernel_with_alloc}. We use first order optimization methods to minimize $\lambda \mapsto f(\lambda)$ since it is convex.

\subsection{Project-Then-Round (PTR) for RKHS designs}\label{sec:RKHS_rounding}

\begin{figure*}[t]
     \centering

    \hfill
     \begin{minipage}[b]{0.32\textwidth}
         \centering
          \includegraphics[width=\linewidth]{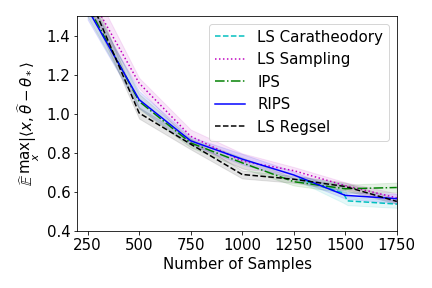}
          \vspace*{-10mm}
              \caption{{
              G-Optimal Experiment}}\label{fig:g_optimal}
     \end{minipage}
    \hfill
               \begin{minipage}[b]{0.32\textwidth}
         \centering
          \includegraphics[width=\linewidth]{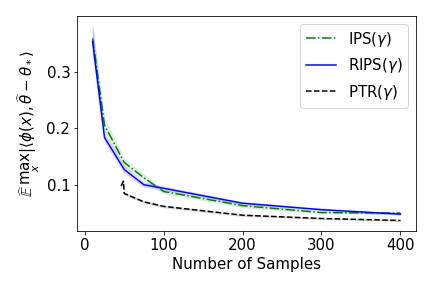}
          \vspace*{-10mm}
               \caption{Kernel Experiment}\label{fig:kernel}
     \end{minipage}
     \hfill
     \begin{minipage}[b]{0.32\textwidth}
         \centering
           \includegraphics[width=\linewidth]{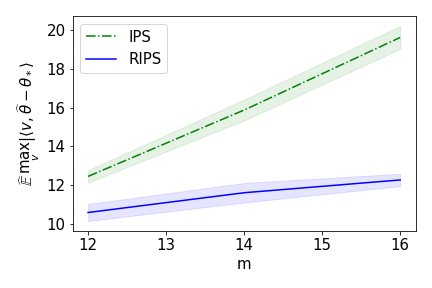}
           \vspace*{-10mm}
           \caption{RIPS vs IPS Experiment 
           }\label{fig:ips}
     \end{minipage}
\end{figure*}

To the best of our knowledge, the {\ours} procedure of  Figure~\ref{alg:RIPS_exp_design_for_RKHS} is novel and should be benchmarked. To design a baseline, we take inspiration from previous works on experimental design in an RKHS.
For instance, \cite{alaoui2015fast} employ a sampling distribution related to statistical leverage scores to construct a sketch of the kernel matrix using a Nystrom approximation. 
The objective in that problem is closest to $V$-optimal design which aims to minimize the sum-squared error $\sum_{x \in \mc{X}} \E[ \langle x, \widehat{\theta} - \theta_* \rangle^2 ]$ (note, our work is concerned with $G$-optimal-like objectives, or worst-case error over $\mc{X}$).
The Nystrom approximation to the kernel matrix effectively projects the problem to a low dimensional sub-space where finite-dimensional rounding techniques like those reviewed in Section~\ref{sct:rounding} can be applied. 
\cite{bach2015equivalence} also relies on a sampling distribution to approximate integrals using kernels with an objective similar to $V$ optimal.

We describe in Algorithm~\ref{alg:PTR} the baseline procedure we call Project-Then-Round (PTR), that employs the finite rounding technique of \cite{allen2017near} described in Section~\ref{sct:rounding}. 

\begin{algorithm}[H] 
  \caption{PTR for Experimental Designs in an RKHS}
  \label{alg:PTR}
\begin{algorithmic}
\State {\bfseries Input:} Finite sets $\mc{X} \subset \R^d$ and $\mc{V} \subset \mc{H}$, feature map $\phi: \R^d \rightarrow \mc{H}$, regularization $\gamma>0$
\State Fix any $\lambda \in \triangle_{\mc{X}}$. 
\State Compute $[K]_{i,j} = k(x_i,x_j) = \langle \phi(x_i), \phi(x_j) \rangle_{\mc{H}}$ the kernel matrix of the set of points $\mc{X} = \{x_1, \ldots, x_n\}$.
\State Consider a decomposition $K = \widehat{\Phi}\widehat{\Phi}^\top$ with $\widehat{\Phi} \in \R^{n\times n}$ such that the rows of $\widehat{\Phi}$ called  $\widehat{\phi}(x_i)\in \R^n$ are used to compute $\widehat{A}(\lambda) = \sum_{i=1}^{n}\lambda_i \widehat{\phi}(x_i) \widehat{\phi}(x_i)^\top$.
\State Diagonalize $\widehat{A}(\lambda)$ as $\widehat{A}(\lambda) = VDV^\top$ with $D$ diagonal matrix with coefficients $(d_1 \geq d_2 \geq \ldots \geq d_n)$.
\State Define the \emph{effective dimension} as
\begin{align*}
    \widetilde{d}(\lambda, \gamma) = \max\{i \in [n]: d_i \geq \gamma \}.
\end{align*}
\State Choose $k = \widetilde{d}(\gamma, \lambda)\in[n]$ and denote $V_k$ as the top $k$ eigenvectors of $\widehat{A}(\lambda)$.
\State Compute the projections $V_k^\top \widehat{\phi}(x_1), \ldots, V_k^\top \widehat{\phi}(x_n) \in \R^k$ \State Use the rounding procedure of \cite{allen2017near} to obtain the desired sparse allocation $\{\widetilde{x}_i\}_{i=1}^\tau$.
\State {\bfseries Return:} $\{\widetilde{x}_i\}_{i=1}^\tau$
\end{algorithmic}
\end{algorithm}

This procedure enjoys the following guarantees.

\begin{theorem}\label{thm:rounding}
Consider the procedure of Algorithm~\ref{alg:PTR}.
If the number of measurements $\tau$ satisfies $\tau = \Omega(\widetilde{d}(\gamma, \lambda))$, then 
\begin{align*}
    &\max_{v \in \mathcal{V}} \|v\|_{\left(\sum_{i=1}^{\tau}\phi(\widetilde{x}_i)\phi(\widetilde{x}_i)^\top + \tau\gamma I\right)^{-1}}^2\\ 
    &\leq \max(2, 1+\epsilon) \max_{v \in \mathcal{V}} \|v\|_{\left(\sum_{x\in \mathcal{X}}\tau\lambda_x \phi(x)\phi(x)^\top + \tau\gamma I\right)^{-1}}^2.
\end{align*}
where $\widetilde{d}(\gamma, \lambda)$ is defined in the algorithm.
\end{theorem}
We refer the reader to the appendix for the proof of Theorem~\ref{thm:rounding}. This procedure performs rounding in a finite dimensional subspace which is a projection of the initial feature space of potentially infinite dimension.
With Theorem~\ref{thm:rounding} one can obtain a guarantee similar to that of Theorem~\ref{thm:robust_estimator_lemma} up to a constant whenever $\tau=\Omega(\widetilde{d}(\lambda, \gamma))$.
Though this effective dimension is rarely the dominating factor in analyses, it is cumbersome to keep around and bound.

\subsection{Empirical evaluation of allocation methods}

We briefly describe illustrative experiments (see the supplementary material for more details).

\textbf{G-optimal design experiment:} We generate $x_1,\ldots, x_n $ by sampling $\tilde{x}_i \sim N(0,\Sigma)$ with $\Sigma_{i,i} = 1 $ if $i \leq d-10$, $\Sigma_{i,i} = .1 $ if $i > d-10$ and all other entries of $\Sigma$ set to $0$. Then, we set $x_i = \tfrac{\tilde{x}_i}{\norm{\tilde{x}_i}}$. We use $\theta_* = \frac{1}{\sqrt{d}}\mathbf{1}$. We set $d = 50$ and $n = \tfrac{d(d+1)}{2}$. We use mirror descent to solve the G-optimal design problem. We compare {\ours} with IPS, Caratheodory's algorithm with the ceiling rounding technique (LS Caratheodory), the rounding technique in \cite{allen2017near} (LS Regsel), and the random sampling approach taken in \cite{rizk2020refined} (LS Sampling). Figure \ref{fig:g_optimal} depicts the results, and shows that {\ours} performs comparably to these other approaches. It also illustrates the shortcomings of the Caratheodory rounding algorithm, which does not return an estimate for $T \leq 1275$, while the other algorithms have already learned nontrivial estimates of $\theta_*$ for much smaller values of $T$.

\textbf{G-optimal design in an RKHS:} We let $\X = \{0,(\tfrac{1}{m})^2, \ldots, (\tfrac{m-1}{m})^2, 1\}$ with $m=500$ and use the RBF kernel $K(x,x^\prime) = \exp(-\frac{\norm{x-x^\prime}^2}{2 \varphi^2})$ with bandwidth parameter $\varphi = 0.025$. Due to this being an infinite dimensional kernel, the ambient dimension for $m$ points is equal to $m$. 
We focus on the regime $T < m$ where standard rounding schemes do not apply and compare PTR with regularization $\gamma$, $\widehat{\theta}_{\ours}(\gamma)$, and $\widehat{\theta}_{IPS(\gamma)} := A^{(\gamma)}(\lambda)^{-1} ( \frac{1}{T} \sum_{t=1}^T x_t y_t )$ where we set $\gamma = 0.005$. Figure \ref{fig:kernel} depicts the results, showing that PTR($\gamma$) does slightly better than IPS($\gamma$) and RIPS($\gamma$), and that all three algorithms have learned non-trivial estimates of $\theta_*$ using hundreds of samples fewer than standard rounding algorithms require to even output an estimate.

\textbf{RIPS vs. IPS:} While IPS has similar performance to RIPS in the two previous experiments, RIPS performs dramatically better in some settings. Let $m \in \mathbb{N}$ and $d = m^2 +m$. Inspired by combinatorial bandits, we consider a setting where the measurement vectors $\X = \{e_1,\ldots, e_d\}$ consist of the standard basis vectors, $\theta_*= -\mathbf{1}$, and the performance metric for an estimator $\widehat{\theta}$ is $\E\sup_{i \in [m^2]}| v_i^\top (\widehat{\theta}-\theta_*)|$ where $v_i = \sum_{j=1}^m e_j + e_{i+m}$. We compare the performance of IPS against {\ours} for $m \in \{12,14,16\}$ and estimate the expected maximum deviation at $T=4m$. Figure \ref{fig:ips} shows that as $m$ grows, the performance of IPS degrades relative to RIPS, reflecting that IPS has large deviations in comparison to our proposed estimator RIPS.

\section{Algorithms for Kernelized Bandits}\label{section:bandits_algo}
We now leverage our proposed {\ours} estimator of Algorithm~\ref{alg:RIPS_exp_design_for_RKHS} for the kernel bandits problem in an elimination style algorithm as introduced in Section~\ref{sec:elim_algorithm_motivation}. In this section we provide different algorithms to solve the regret minimization and pure exploration problems. This section illustrates the benefits of using our {\ours} estimator. In particular, the estimator enables us to design a regret minimization algorithm that trivially supports batching while enjoying state of the art performance in the well-specified setting. In addition, this same algorithm is robust to model misspecification, suffering only linear regret with respect to that approximation error without any prior knowledge on this error (guarantees that, to the best of our knowledge, our novel). Last but not least, applying our {\ours} estimator to pure exploration tasks leads to the first best arm identification provably robust to misspecification.

\subsection{RIPS for Regret minimization}
As introduced in Section~\ref{sct:setting}, our objective is to develop an algorithm that minimizes regret under the general stochastic and misspecified setting (Assumptions 1-3). 
Specifically, when pulling arm $x \in \mc{X}$ at time $t$ we observe a random variable $\mu_x + \xi_t$ where $\xi_t$ is independent, mean-zero noise with variance $\sigma^2$.
We assume there exists a $\theta_* \in \mc{H}$ and known feature map $\phi: \mc{X} \rightarrow \mc{H}$ such that $\max_{x \in \mc{X}} |\mu_x- \langle \theta_*, \phi(x) \rangle| \leq h$ where $h \geq 0$ is unknown to the learner. 
That is, $\mu_x$ is well-approximated by the linear function $\langle \theta_*, \phi(x) \rangle$ but may deviate from it by an amount $h \geq 0$. 
Because of model misspecification in the case when $h >0$, we should not hope to obtain sub-linear regret if we seek a regret bound that grows only logarithmically in $|\mc{X}|$ and polynomial in $d$. 

Algorithm~\ref{alg:regret_robust} is a phased elimination strategy where at each round a (regularized) G-optimal design is performed to minimize the variances of the estimates of all the arms and then arms are discarded if their sub-optimality gap is deemed too large  (under the assumed linear model). 
Due to model misspecification, we should only expect this approach to work until hitting a kind of noise floor defined by the level of misspecification $h$, as suggested from the guarantee from Theorem~\ref{thm:robust_estimator_lemma}. 
The algorithm is a combination of our {\ours} estimator for the RKHS setting and the robust algorithm of \cite{lattimore2020learning}.

\begin{algorithm}[tb]
  \caption{RIPS for Regret Minimization}
  \label{alg:regret_robust}
\begin{algorithmic}
\State {\bfseries Input:} Finite sets $\mc{X} \subset \R^d$ ($|\mc{X}| = n$), feature map $\phi$, confidence level $\delta \in (0, 1)$, regularization $\gamma$, sub-Gaussian parameter $\sigma$, bound on maximum reward $B$. \\
\State Set $\mc{X}_1 \gets \mc{X}, \ell \gets 1$\\
\While{$|\mc{X}_\ell|>1$}
    \State Let $\displaystyle{\lambda}_\ell \in \triangle_{\mc{X}}$ be a minimizer of $f(\lambda; \mc{X}_\ell, \gamma)$ where
        \begin{align*}
            &\displaystyle f(\mc{V}, \gamma) = \inf_{\lambda \in \triangle_{\mc{V}}} f(\lambda; \mc{V}, \gamma) \\
            &= \inf_{\lambda \in \triangle_{\mc{V}}} \max_{y \in \mc{V}}\|\phi(y)\|_{(\sum_{y\in \mc{V}} \lambda_y \phi(y)\phi(y)^{\top} + \gamma I)^{-1}}^2
        \end{align*}
    \State Set $\epsilon_\ell \gets 2^{-\ell}$, $q^{(1)}_\ell\gets c_1\log(|\mc{X}|/\delta)$\\
    \State Set $q^{(2)}_\ell\gets 
    \HighProbConst^2 (B^2 + \sigma^2)\epsilon_\ell^{-2}f(\mc{X}_\ell, \gamma)\log(4\ell^2|\mc{X}|/\delta)$\\
    \State Set $\tau_\ell \gets \left\lceil\max\left\{q^{(1)}_\ell , q^{(2)}_\ell\right\}\right\rceil$\\ 
    \State Use Algorithm~\ref{alg:RIPS_exp_design_for_RKHS} with sets $\mc{X}_\ell$, $\mc{V}_\ell = \phi \circ \mc{X}_\ell$, sampling $\tau_\ell$ measurements $x_1, \ldots, x_{\tau_\ell}$ to get $\{W^{(v)}\}_{v\in\mc{V}_\ell}$.
    \State Set $\displaystyle\widehat{\theta}_\ell := \arg\min_{\theta} \max_{v \in \mc{V}_\ell} \frac{ | \langle \theta, v \rangle - W^{(v)}| }{\|v\|_{(\sum_{x \in \mc{X}} \lambda_{\ell,x} \phi(x) \phi(x)^\top + \gamma I)^{-1}}}$
    \State Update active set:
        \begin{align*}
        \displaystyle 
        \mc{X}_{\ell+1} = \Big\{x \in \mc{X}_\ell, &\max_{x' \in \mc{X}_\ell} \langle \phi(x') - \phi(x), \widehat{\theta}_{\ell} \rangle < 4\epsilon_\ell \Big\}
        \end{align*}
    \State $\ell \gets \ell + 1$ \\
\EndWhile
\State Play unique element of $\mc{X}_\ell$ indefinitely. 
\end{algorithmic}
\end{algorithm}

\begin{theorem}\label{thm:regret_medofmean}
With probability at least $1-\delta$, the regret of Algorithm~\ref{alg:regret_robust} satisfies
\begin{align}\label{regret_bound_robust}
    &\sum_{t=1}^T \mu_x - \mu_{x_t} \lesssim \ c_1 \log(|\mc{X}|/\delta)+ \sqrt{\max_{\mc{V}\subset \mc{X}}f(\mc{V},\gamma)}\times\\
    & \! \times \!\Big(T(h\!+\!\sqrt{\gamma} \|\theta_*\|)\! + \!\sqrt{c_0^2(\sigma^2\!+\!B^2)T\log(|\mc{X}|\log(T)/\delta)} \Big) \nonumber
\end{align}
where $\displaystyle f(\mc{V},\gamma) = \inf_{\lambda \in \triangle_\mc{V}} \sup_{y \in \mc{V}} \|\phi(y)\|^2_{( \sum_{x\in\mc{X}} \lambda_x x x^\top + \gamma I )^{-1}}$.
\end{theorem}

Choosing $\gamma = 1/T$, $\delta=1/T$ yields an expected regret of
\begin{align*}
\E\!\big[\!\sum_{t=1}^T \mu_{x_*} \!-\! \mu_{x_t}\!\big] \!&\leq \!c'\! \sqrt{\!\max_{\mc{V}\subset \mc{X}}f(\mc{V},\!\tfrac{1}{T})} \big( h T\! \!+\! \!\sqrt{\log(  |\mc{X}|T) T} \big)
\end{align*}
where $c' =O( \sqrt{\|\theta_*\|^2 +\sigma^2+B^2})$. Note that the $hT$ term due to model misspecification is comparable to the one in \cite{lattimore2020learning}.
Prior works such as \cite{srinivas2009gaussian,valko2013finitetime} have demonstrated expected regret bounds in the well-specified ($h=0$) setting that scale like $\sqrt{ \gamma_T T \log(|\mc{X}|)}$ where 
\begin{align}\label{eqn:gamma_T}
    \gamma_T := \max_{\lambda \in \triangle_{\mc{X}}} \log\det( T A^{(0)}(\lambda) + \gamma).
\end{align}
where $A^{(0)}(\lambda)$ is defined as in \eqref{eq:A_lambda}. The following lemma shows that our own regret bound is never worse than these results.
\begin{lemma}\label{lmm:variance_vs_information_gain}
Let $\gamma_T$ be defined as in \eqref{eqn:gamma_T}. Then
\begin{align*}
    \max_{\mc{V}\subset \mc{X}}\!f(\mc{V},\!\tfrac{1}{T}) \! &= \!\max_{\mc{V}\subset \mc{X}}\!\inf_{\lambda\in\triangle_\mc{V}}\!\sup_{y\in\mc{V}}\!\|\phi(y)\|^2_{A^{(1/T)}(\lambda)^{-1}} \!\leq\! \frac{3}{2}\gamma_T\;.
\end{align*}
\end{lemma}
The quantity $f(\mc{X},\gamma)$ can also be bounded by a more interpretable form:
\begin{lemma}\label{lem:effective_dim_trace}
If $\displaystyle \!f(\mc{X},\gamma)\!=\!\inf_{\lambda \in \triangle_{\mc{X}}}\! \max_{y \in \mc{X}}\!\|\phi(y)\|_{(A(\lambda) + \gamma I)^{-1}}^2$ then
\begin{align*}
    f(\mc{X},\gamma)&\leq\tr\left(A(\lambda_D^*)(A(\lambda_D^*) + \gamma I)^{-1}\right)\\
    &= \tr\left(K_{\lambda_D^*} (K_{\lambda_D^*} + \gamma I)^{-1}\right)
\end{align*}
where $\lambda_D^* \in \arg\max_{\lambda \in \triangle_{\mc{X}}}\log\det\left(A^{(\gamma)}(\lambda)\right)$.
\end{lemma}
Notably, the RHS of Lemma~\ref{lem:effective_dim_trace} is the notion of effective dimension that appears in  \cite{alaoui2015fast,derezinski2020bayesian}. 

\subsection{RIPS for Pure Exploration}
We consider a slight generalization of the pure exploration setting introduced in Section~\ref{sct:setting}.
Fix finite sets $\mc{X} \subset \R^d$ and $\mc{Z} \subset \R^d$. 
We may have $\mc{X}=\mc{Z}$ but there are interesting cases in which $\mc{X} \neq \mc{Z}$ including combinatorial bandits and recommendation tasks \cite{fiez2019}.
We say a $z \in \mc{Z}$ is $\epsilon$-good if $\mu_z \geq \max_{z' \in \mc{Z}} \mu_{z'} - \epsilon$.
In the pure exploration game, for $\epsilon >0$ and $\delta \in (0,1)$ the player seeks to identify an $\epsilon$-good arm by taking as few measurements in $\mc{X}$ as possible. 
Just as in regret minimization games, we assume that when the player at time $t$ plays $x_t \in \mc{X}$ she observes $y_t = \mu_{x_t} + \xi_t$ where $\xi_t$ is independent mean-zero noise with variance $\sigma^2$.
Finally, we assume the existence of a $\theta_* \in \mc{H}$ such that 
\begin{align*}
    \max\!\left\{ \max_{z \in \mc{Z}}|\mu_z -  \langle \theta_*, \phi(z) \rangle| , \max_{x \in \mc{X}}|\mu_x -  \langle \theta_*, \phi(x) \rangle| \!\right\} \!\leq h 
\end{align*}
for some $h \geq 0$ that is \emph{unknown} to the player. 

Consider the elimination style algorithm of Algorithm~\ref{alg:bai_med}.
The algorithm is a combination of our {\ours} procedure and the algorithm of \cite{fiez2019}. 
While the algorithm is inspired by \cite{fiez2019}, their analysis only holds in the well-specified setting ($h=0$), hence a new proof technique was necessary to achieve the following result for general $h \geq0$. 

\begin{algorithm}[tb]
  \caption{RIPS for Pure Exploration}
  \label{alg:bai_med}
\begin{algorithmic}
\State {\bfseries Input:} Finite sets $\mc{X} \subset \R^d$, $\mc{Z} \subset \R^d$, feature map $\phi$, confidence level $\delta \in (0, 1)$, regularization $\gamma$, sub-Gaussian parameter $\sigma$, bound on maximum reward $B$, bound on the misspecification noise $h$.\\
\State Let $\mc{Z}_1 \gets \mc{Z}, \ell\gets 1$ \\
\While{$|\mc{Z}_\ell|>1$}
    \State Let $\displaystyle{\lambda}_\ell \in \triangle_{\mc{X}}$ be a minimizer of $f(\lambda; \mc{Z}_\ell; \gamma)$ where
        \begin{align*}
            &f(\mc{V}; \gamma) = \inf_{\lambda \in \triangle_{\mc{X}}} f(\lambda; \mc{V}; \gamma) \\
            &\!=\! \!\inf_{\lambda \in \triangle_{\mc{X}}} \!\max_{v,v' \in \mc{V}}\!\|\phi(v)\!-\!\phi(v')\|_{(\sum_{x\in \mc{X}}\! \lambda_x \phi(x)\phi(x)^{\top}\! + \gamma I)^{-1}}^2
        \end{align*}
    \State Set $\epsilon_\ell \gets 2^{-\ell}$, $q^{(1)}_\ell\gets c_1\log(|\mc{Z}|/\delta)$\\
    \State Set $q^{(2)}_\ell\gets 
    \HighProbConst^2 \epsilon_\ell^{-2} f(\mc{Z}_\ell;\gamma) (B^2 + \sigma^2) \log(2\ell^2|\mc{Z}|^2/\delta)$\\
    \State Set $\tau_\ell \gets \left\lceil\max\left\{q^{(1)}_\ell , q^{(2)}_\ell\right\}\right\rceil$\\ 
    \State Use Algorithm~\ref{alg:RIPS_exp_design_for_RKHS} with sets $\mc{X}$, $\mc{V}_\ell = \phi \circ \mc{Z}_\ell - \phi \circ \mc{Z}_\ell$, sampling $\tau_\ell$ measurements $x_1, \ldots, x_{\tau_\ell}$ to get $\{W^{(v)}\}_{v\in\mc{V}_\ell}$.
    \State Set $\displaystyle\widehat{\theta}_\ell\! := \!\arg\min_{\theta} \max_{v \in \mc{V}_\ell} \frac{ | \langle \theta, v \rangle - W^{(v)}| }{\|v\|_{(\sum_{x \in \mc{X}}\! \lambda_{\ell, x} \phi(x) \phi(x)^\top\! + \gamma I)^{-1}}}$
    \State $\displaystyle \mc{Z}_{\ell+1} \!= \!\big\{z \in \mc{Z}_\ell  : \max_{z' \in \mc{Z}_\ell} \langle \phi(z') \!- \!\phi(z), \widehat{\theta}_{\ell} \rangle \!\leq \!2 \epsilon_\ell  \big\}$\\
    \State $\ell \gets \ell + 1$ \\
\EndWhile
\State {\bfseries Output:} $\mc{Z}_{\ell}$
\end{algorithmic}
\end{algorithm}
\begin{theorem}\label{thm:samplecomplexity_medofmean}
With $\displaystyle z_* \in \arg\max_{z \in \mc{Z}} \langle z, \theta_* \rangle$,
fix any $\epsilon \geq \bar\epsilon$ where
\begin{align*}
    \bar{\epsilon} &=  8 \min\{ \epsilon \geq 0 : 4(\sqrt{\gamma} \|\theta_* \|_2 + h)(2+ \sqrt{g(\epsilon)}) \leq \epsilon \}, \\
    g(\epsilon) \!&=\! \inf_{\lambda \in \triangle_{\mathcal{X}}} \sup_{z \in \mc{Z} : \langle \theta_*, \phi(z_*)\!-\!\phi(z) \rangle \leq \epsilon} \|\phi(z_*) - \phi(z)\|^2_{A^{(\gamma)}(\lambda)^{-1}}
\end{align*}
Then with probability at least $1-\delta$, once the algorithm has taken at least $\tau$ samples where $\tau = \widetilde{O}( c_1 \log(|\mc{Z}|/\delta)+ \log(\epsilon^{-1} ) \HighProbConst^2 (B^2 + \sigma^2) \log(|\mc{Z}|/\delta) \rho^*(\gamma,\epsilon) )$ we have that $\mu_{\widehat{z}} \geq \max_{z' \in \mc{Z}} - \epsilon$ where $\widehat{z}$ is any arm in the set $\mc{Z}_\ell$ under consideration after $\tau$ pulls and 
\begin{align}\label{rho^*}
    \rho^*(\gamma,\epsilon) \!&=\! \inf_{\lambda \in \Delta_{\mathcal{X}}}\sup_{z\in \mathcal{Z}}\frac{\|\phi(z_*) \!-\! \phi(z)\|^2_{A^{(\gamma)}(\lambda)^{-1}}}{\max\{\epsilon^2,\langle \theta_*, \phi(z_*)\!-\!\phi(z) \rangle^2)\}}.
\end{align}
\end{theorem}
Note that if $\mc{X}=\mc{Z}$ we have
\begin{align*}
    g(\epsilon) \!&=\! \inf_{\lambda \in \triangle_{\mathcal{X}}} \sup_{z \in \mc{X} : \langle \theta_*, \phi(z_*)\!-\!\phi(z) \rangle \leq \epsilon}\! \|\phi(z_*) - \phi(z)\|^2_{A^{(\gamma)}(\lambda)^{-1}} \\
    &\leq 4 \inf_{\lambda \in \triangle_{\mathcal{X}}} \sup_{x \in \mc{X}} \|\phi(x) \|^2_{A^{(\gamma)}(\lambda)^{-1}} \\
    &\leq 4 \tr( (A(\lambda_D^*) + \gamma I)^{-1} A(\lambda_D^*) )
\end{align*}
where the last line follows from Lemma~\ref{lem:effective_dim_trace}.
This means $\bar\epsilon$, the limit on how well one can estimate the maximizing arm, satisfies $\bar\epsilon \lesssim  (\gamma \|\theta_*\| + h) \tr( (A(\lambda_D^*) + \gamma I)^{-1} A(\lambda_D^*) )$. 
Thus, if we seek an $\epsilon$-good arm, we should choose $\gamma$ to make this right hand side less than $\epsilon$. 
Note that $\gamma=0$ and $h=0$ implies $\bar\epsilon = 0$.
If $\phi\equiv \text{identity}$ so that $\mc{H}=\R^d$, $h=0$, and $\gamma=0$ then  the sample complexity of Theorem~\ref{thm:samplecomplexity_medofmean} is known to be optimal up to log factors to identify the very best arm (assuming it is unique) relative to any $\delta$-correct algorithm over $\theta_* \in \R^d$ \cite{soare2014best,fiez2019}.

\subsection{Comparing to the alternative baseline procedure}
In Section~\ref{sec:RKHS_rounding} we proposed a natural alternative to our {\ours} procedure for experimental design in an RKHS.
This PTR baseline leveraged the fact that the added regularization $\gamma >0$ effectively made many directions irrelevant.
Thus, it projected the problem to a low dimensional subspace where it could apply any of the standard rounding techniques for finite dimensions described in Section~\ref{sct:rounding}.
The dimension of this subspace, denoted $\widetilde{d}$, scales like the number of eigenvalues of $\sum_{x\in \mathcal{X}}\lambda^*_x \phi(x)\phi(x)^\top$ that are greater than $\gamma$ where $\lambda^* \in \arg\min_{\lambda \in\triangle_{\mc{X}}}\max_{v \in \mathcal{V}} \|v\|_{\left(\sum_{x\in \mathcal{X}}\lambda_x \phi(x)\phi(x)^\top + \gamma  I\right)^{-1}}^2$.
Any standard rounding algorithm would then require the number of samples taken from the design to be at least $\widetilde{d}$. 
Relative to our results, this inflates our regret bound and sample complexity by an additive factor of $\widetilde{d}$ scaled by some problem-dependent $\log$ factors. 
Algebra shows that $\widetilde{d} \leq  2 \tr( (A(\lambda_D^*) + \gamma I)^{-1} A(\lambda_D^*) )$.
Though for regret this is a lower order term, for pure-exploration with $\mc{X} \neq \mc{Z}$, this term may potentially dominate the sample complexity because it does not capture the interplay between the geometry of $\mc{X}$ and $\mc{Z}$. 
Fortunately, our {\ours} procedure demonstrates it is unnecessary and avoids it.

\section{Related work}
There exist excellent surveys of experimental design from both a statistical and computational perspective \cite{pukelsheim2006optimal,atkinson2007optimum,todd2016minimum}.
Our work is particularly interested in the task of converting a continuous design into a discrete allocation of $T$ measurements.
We reviewed a number of works in Section~\ref{sct:rounding} for completing this task in finite dimensions. 
To move to an RKHS setting we considered a regularized design objective which is also known as Bayesian experimental design \cite{chaloner1995bayesian,allen2017near,derezinski2020bayesian}. 
While most Bayesian experimental design works assume a low-dimensional ambient space and use simple rounding, one exception is the work of \cite{alaoui2015fast} that performs experimental design in an RKHS for a different design objective, which inspired our project-then-round procedure described of Section~\ref{sec:RKHS_rounding}.
And very recently, \cite{derezinski2020bayesian} proposed a method of sampling from a determinantal point process (DPP) and showed that they can approximate many continuous experimental design objectives up to a constant factor if $T \gtrsim d_{eff} := \tr( (A(\lambda_D^*)+\gamma I)^{-1} A(\lambda_D^*) )$ with $\lambda_D^*$ defined in Lemma~\ref{lem:effective_dim_trace}.
However, according to Table~1 of \cite{derezinski2020bayesian} the method may not apply to $G$-optimal-like objectives\footnote{Our Theorem~\ref{thm:rounding} with the fact  $\widetilde{d} \leq 2 d_{eff}$ suggests $k$ only needs to be at least $\widetilde{d}$ for $G$-like objectives, which adds to their table.}, which is the primary  objective of our work. 
To our knowledge, our proposed {\ours} method is novel in that its performance is directly comparable to the continuous design without requiring a minimum number of measurements with some dependence on the (effective) dimension.  
However, our method does require the number of measurements to exceed $\log(|\mc{V}|)$.
While we leveraged experimental design techniques for kernel bandits, many prior works were able to obtain regret bounds and pure-exploration results using other methods.

\textbf{Kernel bandits}
In the well-specified setting ($h=0$) \cite{srinivas2009gaussian} propose a UCB style algorithm \cite{auer2002finite} for the RKHS setting. Independently, \cite{grunewalder2010regret} developed similar methods for minimizing simple regret. 
\cite{srinivas2009gaussian} established a regret bound of $\sqrt{T}( \|\theta_*\| \sqrt{\gamma_T} + \gamma_T)$ where $\gamma_T$ is defined in \eqref{eqn:gamma_T}. 
\cite{valko2013finitetime} proposed another UCB variant to obtain a regret bound that scales just as $\|\theta_*\| \sqrt{p_T T}$ where $p_T$ is an algorithm-dependent constant that can be upper bounded by $\gamma_T$, thus improving \cite{srinivas2009gaussian}.
We recall that our own regret bound of Theorem~\ref{thm:regret_medofmean} scales no worse than $\|\theta_*\| \sqrt{\gamma_T T}$ using Lemma~\ref{lmm:variance_vs_information_gain}, thus matching state of the art.
\cite{chowdhury2017kernelized} offer improvements in regret over GP-UCB  when the action space is infinite. 
We also note that our algorithm naturally allows batch querying, a property that UCB-like algorithms achieve only through inelegant means  \cite{desautels2012parallelizing,wu2018parallel}.

\textbf{Misspecified models}
Our approach to misspecified models draws inspiration from  \cite{lattimore2020learning} which addresses linear bandits in finite dimensions. 
Their regret bound scales quadratically in the ambient dimension due to rounding effects.
Our {\ours} procedure extends this work to an RKHS. 
The misspecified model setting is related to the corrupted setting where an adversary can choose to corrupt the observed reward by $c_t$ in each round $t$. Any algorithms for this adversarial setting can also be used to solve kernelized multi-armed bandit in the misspecified setting with total amount of corruption equal to at most $C_T=\sum_{t=1}^T c_t =hT$. 
Using this reduction, the regret bound for the corrupted setting  of \cite{bogunovic2020corruptiontolerant} scales like $C_T \sqrt{\gamma_T T}$. 
Unfortunately, if we take $C_T = hT$ this bound is vacuous. 
Whether robust algorithms like \cite{gupta2019better} can be extended to our kernel bandit setting is an open question. 
Concurrently, \cite{lee2021achieving} independently proposed a very similar estimator and algorithm for the related task of solving adversarial bandits.

\textbf{Constrained linear bandits}
If we assumed that $\| \theta_* \|_2 \leq R$ for some explicit, known $R > 0$ then this setting is known as constrained linear bandits, tackled in \cite{degenne2020gamification} for the pure-exploration and \cite{tirinzoni2020asymptotically} for the regret setting, respectively. There, a lower bound on the sample complexity of identifying the best arm can be computed.
The lower bound is  $\inf_{\lambda\in \Delta_{\mathcal{X}}}\!\sup_{x' \neq x_*}\!\!\inf_{\gamma \geq 0}\! \!\frac{2\|x'-x_*\|^2_{(A(\lambda) + \gamma\! I\!)^{-1}}}{\max\{(x'\!-x_*)\!^\top \!(\!A(\lambda) + \gamma I)^{-1}\!A(\lambda) \theta_*\!, 0\}^2}$, which is close to our upper bound $\rho^*$ from equation \ref{rho^*}. \\
\cite{degenne2020gamification} propose an algorithm with an asymptotic upper bound in the sense that as $\delta \rightarrow 0$, the dominant term matches the lower bound.
However, 
while \cite{degenne2020gamification} and \cite{tirinzoni2020asymptotically} are tight asymptotically, they suffer from large sub-optimal dependencies on problem-specific parameters.

\section{Conclusion}
In this paper, we have brought to the non-parametric learning setting an estimator that relies on continuous designs while enjoying state of the art - theoretical and experimental - guarantees for both the well-specified and the misspecified settings. We leveraged this estimator in a novel elimination style algorithm for kernel bandits. For the most part we have ignored computation. However, the computational cost of the {\ours} estimator scales \emph{linearly} in $|\mc{V}|$. An interesting avenue of research is designing an estimator that leverages multi-dimensional robust mean estimation that has the same properties as {\ours} but has \emph{no} dependence on $|\mc{V}|$. Such an estimator would be of considerable interest in problems such as combinatorial bandits where $|\mc{V}|$ is potentially exponential in the dimension (e.g., see \cite{katz2020empirical,wagenmaker2021experimental}).

\newpage 
\bibliographystyle{plainnat}
\bibliography{bibliography.bib}

\newpage
\clearpage
\newpage
\onecolumn
\appendix
\section*{Outline} 
The appendix is organized as follows. We first provide the proofs for the concentration bound of {\ours} (Theorem~\ref{thm:robust_estimator_lemma}), the computation of the inverse of the bilinear form (Lemma~\ref{lmm:compute_kernel_with_alloc}), the guarantees of the PTR procedure (Theorem~\ref{thm:rounding}), the regret bound of the {\ours} regret minimization algorithm (Theorem~\ref{thm:regret_medofmean}), the sample complexity of the {\ours} pure exploration algorithm (Theorem~\ref{thm:samplecomplexity_medofmean}). We also establish the regret bound and the sample complexity guarantees of the PTR procedure. Then, we provide the proofs of the comparison of our variance term $f(\mc{X}, 1/T)$ with the information gain of \cite{srinivas2009gaussian} (Lemma~\ref{lmm:variance_vs_information_gain}) and with the effective dimension of \cite{alaoui2015fast} (Lemma~\ref{lem:effective_dim_trace}) and prove a corollary of Theorem 1 of \cite{degenne2020gamification}. Last, we complete the details of the experiments.

\section{Concentration of {\ours}, Proof of Theorem~\ref{thm:robust_estimator_lemma}}\label{section:proof_concentration_robust_est}

\begin{proof}
First note that 
\begin{align*}
     \max_{v \in \mc{V}} \frac{| \langle \widehat{\theta}, v \rangle - \langle \theta_*, v \rangle|}{\|v\|_{A^{(\gamma)}(\lambda)^{-1}}} &= \max_{v \in \mc{V}} \frac{| \langle \widehat{\theta}, v \rangle - W^{(v)} + W^{(v)} - \langle \theta_*, v \rangle|}{\|v\|_{A^{(\gamma)}(\lambda)^{-1}}} \\
     &\leq \max_{v \in \mc{V}} \frac{| \langle \widehat{\theta}, v \rangle - W^{(v)}|}{\|v\|_{A^{(\gamma)}(\lambda)^{-1}}} + \max_{v \in \mc{V}} \frac{| W^{(v)} - \langle \theta_*, v \rangle|}{\|v\|_{A^{(\gamma)}(\lambda)^{-1}}} \\
     &= \min_{\theta} \max_{v \in \mc{V}} \frac{ | \langle \theta, \phi(v )\rangle - W^{(v)}| }{\|v\|_{A^{(\gamma)}(\lambda)^{-1}}} + \max_{v \in \mc{V}} \frac{| W^{(v)} - \langle \theta_*, v \rangle|}{\|v\|_{A^{(\gamma)}(\lambda)^{-1}}} \\
     &\leq 2 \max_{v \in \mc{V}} \frac{ | \langle \theta_*, v \rangle - W^{(v)}| }{\|v\|_{A^{(\gamma)}(\lambda)^{-1}}}
\end{align*}
which completes the second part of the lemma, so it suffices to show that each $| \langle \theta_*, v \rangle - W^{(v)}|$ is small.

We begin by bounding the variance of $v^\t A^{(\gamma)}(\lambda)^{-1} \phi(x_t) y_t$ for any $t \in \mathbb{N}$ which is necessary to use the robust estimator. Note that
\begin{align*}
\mathbb{V}\text{ar}( v^\t A^{(\gamma)}(\lambda)^{-1} \phi(x_t) y_t ) &= \E[ ( v^\t A^{(\gamma)}(\lambda)^{-1} \phi(x_t) y_t )^2 ] - \E[ v^\t A^{(\gamma)}(\lambda)^{-1} \phi(x_t) y_t ]^2 
\end{align*}
which means we can drop the second term to bound the variance by
\begin{align*}
    \E[ \left( v^\t A^{(\gamma)}(\lambda)^{-1} \phi(x_t) y_t \right)^2 ] &= \E[ \left( v^\t A^{(\gamma)}(\lambda)^{-1} \phi(x_t) (  \phi(x_t)^\top \theta_* + \eta_t + \xi_t ) \right)^2 ] \\
    &= \E[ \left( v^\t A^{(\gamma)}(\lambda)^{-1} \phi(x_t) (  \phi(x_t)^\top \theta_* + \eta_t) \right)^2 ] + \E[ \left( v^\t A^{(\gamma)}(\lambda)^{-1} \phi(x_t)  \right)^2 \xi_t^2 ] \\
    &\leq B^2 \E[ \left( v^\t A^{(\gamma)}(\lambda)^{-1} \phi(x_t)  \right)^2 ] + \sigma^2 \E[ \left( v^\t A^{(\gamma)}(\lambda)^{-1} \phi(x_t)  \right)^2 ] \\
    &= (B^2 + \sigma^2) \E[ v^\t A^{(\gamma)}(\lambda)^{-1} \phi(x_t) \phi(x_t)^\top  A^{(\gamma)}(\lambda)^{-1} v ] \\
    &\leq (B^2 + \sigma^2) \|v\|_{A^{(\gamma)}(\lambda)^{-1}}^2.
\end{align*}

Recalling that
\begin{align*}
&|\widehat{\mu}( \{ v^\t A^{(\gamma)}(\lambda)^{-1} \phi(x_t) y_t \}_{t=1}^T ) - \E[ v^\t A^{(\gamma)}(\lambda)^{-1} \phi(x_1) y_1 ]| \leq \HighProbConst \sqrt{\mathbb{V}\text{ar}(v^\t A^{(\gamma)}(\lambda)^{-1} \phi(x_1) y_1 ) \frac{\log(\tfrac{2}{\delta})}{T}} 
\end{align*}
we have 
\begin{align*}
    | \langle \theta_*, v \rangle - W^{(v)}| &= | \langle \theta_*, v \rangle - \E[ v^\t A^{(\gamma)}(\lambda)^{-1} \phi(x_1) y_1 ] + \E[ v^\t A^{(\gamma)}(\lambda)^{-1} \phi(x_1) y_1 ] - W^{(v)}| \\
    &\leq | \langle \theta_*, v \rangle - \E[ v^\t A^{(\gamma)}(\lambda)^{-1} \phi(x_1) y_1 ] | + |\widehat{\mu}( \{ v^\t A^{(\gamma)}(\lambda)^{-1} \phi(x_t) y_t \}_{t=1}^T ) - \E[ v^\t A^{(\gamma)}(\lambda)^{-1} \phi(x_1) y_1 ]|.
\end{align*}
We now recall that $y_t = \langle \phi(x_t), \theta_* \rangle + \xi_t + \eta_{x_t}$ where $\xi_t$ is a mean-zero, independent random variable with variance $\sigma^2$, and $|\eta_{x_t}| \leq h$.
Thus, 
\begin{align*}
    | \langle \theta_*, v \rangle - \E[ v^\t A^{(\gamma)}(\lambda)^{-1} \phi(x_1) y_1 ] | &= | \langle \theta_*, v \rangle - \E[ v^\t A^{(\gamma)}(\lambda)^{-1} \phi(x_1)\phi(x_1)^\top \theta_* ] - \E[ v^\t A^{(\gamma)}(\lambda)^{-1} \phi(x_1) \eta_{x_1} ] | \\
    &\leq | \langle \theta_*, v \rangle - \E[ v^\t A^{(\gamma)}(\lambda)^{-1} \phi(x_1)\phi(x_1)^\top \theta_* ]| + |\E[ v^\t A^{(\gamma)}(\lambda)^{-1} \phi(x_1) \eta_{x_1} ] |
\end{align*}
which we bound separately. 
Firstly,
\begin{align*}
     | \langle \theta_*, v \rangle - \E[ v^\t A^{(\gamma)}(\lambda)^{-1} \phi(x_1) \phi(x_1)^\top \theta_* ]| &=  | \langle \theta_*, v \rangle -  v^\t A^{(\gamma)}(\lambda)^{-1} A(\lambda) \theta_* |  \\
     &=  \gamma |  v^\t A^{(\gamma)}(\lambda)^{-1} \theta_* | \\
     &=  \gamma^{1/2} |  v^\t  (A(\lambda) + \gamma I)^{-1/2} (A(\lambda)/\gamma +  I)^{-1/2} \theta_* | \\
     &\leq  \gamma^{1/2} |  v^\t  (A(\lambda) + \gamma I)^{-1/2} \theta_* | \\
     &\leq  \gamma^{1/2} \| v \|_{A^{(\gamma)}(\lambda)^{-1}} \|\theta_*\|
\end{align*}
and secondly,
\begin{align*}
    |\E[ v^\t A^{(\gamma)}(\lambda)^{-1} \phi(x_1) \eta_{x_1} ] | &\leq \E[| v^\t A^{(\gamma)}(\lambda)^{-1} \phi(x_1) \eta_{x_1}| ] \\
    &\leq h \sqrt{ \E[| v^\t A^{(\gamma)}(\lambda)^{-1} \phi(x_1) |^2 ]} \\
    &= h \sqrt{ v^\t A^{(\gamma)}(\lambda)^{-1} A(\lambda) A^{(\gamma)}(\lambda)^{-1} v } \\
    &\leq h \| v \|_{A^{(\gamma)}(\lambda)^{-1}}.
\end{align*}
Thus, putting it all together we have
\begin{align*}
    | \langle \theta_*, v \rangle - W^{(v)}| &\leq ( \sqrt{\gamma} \|\theta_* \|_2 + h + \HighProbConst\sqrt{(B^2 + \sigma^2) \frac{\log(2/\delta) }{T}} ) \| v \|_{A^{(\gamma)}(\lambda)^{-1}}.
\end{align*}
Union bounding over all $v \in \mc{V}$ completes the proof.
\end{proof}

\section{Inverses and bilinear forms, Proof of Lemma~\ref{lmm:compute_kernel_with_alloc}}
\begin{proof}[Proof of Proposition~\ref{lmm:compute_kernel_with_alloc}]
The following manipulations are well-known, but we include them from completeness.
Define
\begin{align*}
    \Phi := [\phi(x_1)^\top, \ldots, \phi(x_\tau)^\top]^\top
\end{align*}
Holds 
\begin{align*}
    \Phi^\top (\Phi\Phi^\top + \gamma I) = (\Phi^\top \Phi + \gamma I) \Phi^\top
    \;.
\end{align*}
And thus
\begin{align*}
    (\Phi^\top \Phi + \gamma I)^{-1} \Phi^\top = \Phi^\top (\Phi\Phi^\top + \gamma I)^{-1}\;.
\end{align*}
Now we use the expansion
\begin{align*}
    \left(\Phi^\top \Phi + \gamma I\right)a = \Phi^\top \Phi a + \gamma a
\end{align*}
to write
\begin{align*}
    a &= \left(\Phi^\top \Phi + \gamma I\right)^{-1}\left(\Phi^\top  \Phi a + \gamma a\right)\\
    &= \left(\Phi^\top \Phi + \gamma I\right)^{-1}\Phi^\top  \Phi a + \left(\Phi^\top \Phi + \gamma I\right)^{-1}\gamma a\\
    &= \Phi^\top \left(\Phi \Phi^\top + \gamma I\right)^{-1} \Phi a + \gamma\left(\Phi^\top \Phi + \gamma I\right)^{-1} a\;.
\end{align*}
Then multiplying on the left side by $b^\top$ leads to 
\begin{align*}
    b^\top a = b^\top\Phi^\top \left(\Phi \Phi^\top + \gamma I\right)^{-1} \Phi a + \gamma b^\top\left(\Phi^\top \Phi + \gamma I\right)^{-1} a\;.
\end{align*}
So
\begin{align*}
    b^\top\left(\sum_{x'\in \mathcal{X}} \phi(x')\phi(x')^\top + \gamma I\right)^{-1} a
    &= \frac{1}{\gamma} b^\top a - \frac{1}{\gamma}b^\top \Phi^\top \left(\Phi \Phi^\top + \gamma I\right)^{-1} \Phi a\\ 
    &= \frac{1}{\gamma} a^\top b - \frac{1}{\gamma}k(a)^\top \left(K + \gamma I\right)^{-1} k(b)\;.
\end{align*}
We now simply repeat with the same calculations with
\begin{align*}
    \Phi_\lambda := [\sqrt{\lambda_1}\phi(x_1)^\top, \ldots, \sqrt{\lambda_\tau}\phi(x_\tau)^\top]^\top\;,
\end{align*}
\begin{align*}
    K_\lambda = \Phi_\lambda \Phi_\lambda^\top = \left[\sqrt{\lambda_i}\sqrt{\lambda_j}\phi(x_i)^\top\phi(x_j)\right]_{1\leq i,j\leq \tau}\;,
\end{align*}
and 
\begin{align*}
    k_\lambda(x) := \Phi_\lambda\phi(x) \in \mathbb{R}^{\tau}\;.
\end{align*}
\end{proof}

\section{Guarantees of the PTR procedure, Proof of Theorem~\ref{thm:rounding}}\label{proofs:rounding}
We establish the proof in a finite dimension case where $\phi$ is the identity map and then extend it to any feature map $\phi$. In both cases, we fix $\mc{X} \subset \R^d$ and consider $\lambda \in \triangle_\mc{X}$ to be the design we wish to round.
\subsection{Finite dimension}\label{finite_dim_rounding}
\begin{lemma}\label{finite_dim}
Let $V D V^\top$ be the eigenvalue decomposition of the matrix $\sum_{x \in \mc{X}} \lambda_x x x^\top$, and denote $D=\text{diag}(d_1,\dots,d_d)$. For any $z\in\mc{V}$, as long as $\tau = \Omega(k/\epsilon)$, we can find an allocation $\{\widetilde{x}_i\}_{i=1}^\tau \subset \mc{X}$ such that 
\begin{align*}
    z^\top \left( \sum_{i=1}^\tau \widetilde{x}_i \widetilde{x}_i^\top + \tau\gamma I_d \right)^{-1} z \leq \max\{1+\epsilon,2\} z^\top \left( \tau\sum_{x \in \mc{X}} \lambda_x x x^\top + \tau\gamma I_d \right)^{-1} z \;,
\end{align*}
where we defined $k = \max\{ i : d_i \geq \gamma \}$. 
\end{lemma}
\begin{proof}
Start by also denoting $V = [v_1,\dots,v_d]$.
Then 
\begin{align*}
    z^\top ( \tau\sum_{x \in \mc{X}} \lambda_x x x^\top + \tau\gamma I)^{-1} z &= z^\top ( \tau VDV^\top + \tau \gamma I)^{-1} z \\
    &= z^\top ( \tau VDV^\top + \tau\gamma V V^\top)^{-1} z\\
    &= z^\top V (\tau D + \tau\gamma I)^{-1} V^\top z \\
    &= z^\top \left(\sum_{i=1}^d v_i v_i^\top \frac{1}{\tau d_i + \tau\gamma}\right) z
\end{align*}
Now, for any $k = \max\{ i : d_i \geq \gamma \}$ we have
\begin{align*}
    z^\top \left(\sum_{i=1}^d v_i v_i^\top \frac{1}{\tau d_i + \tau\gamma}\right) z &= z^\top \left(\sum_{i=1}^k v_i v_i^\top \frac{1}{\tau d_i + \tau\gamma}\right) z + z^\top \left(\sum_{i=k+1}^d v_i v_i^\top \frac{1}{\tau d_i + \tau\gamma}\right) z \\
    &\geq z^\top \left(\sum_{i=1}^k v_i v_i^\top \frac{1}{\tau d_i + \tau\gamma}\right) z + \frac{1}{2} z^\top \left(\sum_{i=k+1}^d v_i v_i^\top \frac{1}{\tau\gamma}\right) z \\
    &= (V_k^\top z)^\top V_k^\top \left(\sum_{i=1}^k v_i v_i^\top \frac{1}{\tau d_i + \tau\gamma}\right) V_k (V_k^\top z) + \frac{1}{2} z^\top \left(\sum_{i=k+1}^d v_i v_i^\top \frac{1}{\tau\gamma}\right) z \\
    &= (V_k^\top z)^\top V_k^\top \left( \tau\sum_{x \in \mc{X}} \lambda_x x x^\top + \tau\gamma I_d\right)^{-1} V_k (V_k^\top z) + \frac{1}{2} z^\top \left(\sum_{i=k+1}^d v_i v_i^\top \frac{1}{\tau\gamma}\right) z \\
    &= (V_k^\top z)^\top \left( \tau\sum_{x \in \mc{X}} \lambda_x (V_k^\top x) (V_k^\top x)^\top + \tau\gamma I_k \right)^{-1} (V_k^\top z) + \frac{1}{2} z^\top \left(\sum_{i=k+1}^d v_i v_i^\top \frac{1}{\tau\gamma}\right) z.
\end{align*}
where we denote $V_k$ and $V_{-k}$ as the top $k$ and bottom $d-k$ eigenvectors, respectively. But now we notice that for this first term, we have $\max\{ \text{dim}(\text{span}(\{ V_k^\top z \}_{z \in \mc{V}})), \text{dim}(\text{span}(\{ V_k^\top x \}_{x \in \mc{X}})) \} \leq k$ which now means that thanks to \cite{allen2017near} we can find an allocation $\{\widetilde{x}_i\}_{i=1}^\tau \subset \mc{X}$ such that 
\begin{align*}
    (V_k^\top z)^\top \left(\tau \sum_{x \in \mc{X}} \lambda_x (V_k^\top x) (V_k^\top x)^\top + \tau\gamma I_k \right)^{-1} (V_k^\top z) \geq \frac{1}{1+\epsilon}(V_k^\top z)^\top \left(\sum_{i=1}^\tau (V_k^\top \widetilde{x}) (V_k^\top \widetilde{x})^\top + \tau\gamma I_k \right)^{-1} (V_k^\top z)
\end{align*}
as long as $\tau = \Omega(k/\epsilon)$.
Putting it altogether we have
\begin{align*}
    &z^\top \left( \sum_{i=1}^\tau \widetilde{x}_i \widetilde{x}_i^\top + \tau\gamma I_d \right)^{-1} z \\
    &=(V_k^\top z)^\top \left(\sum_{i=1}^\tau (V_k^\top \widetilde{x}_i) (V_k^\top \widetilde{x}_i)^\top + \tau\gamma V_k^\top V_k\right)^{-1} (V_k^\top z) + (V_{-k}^\top z)^\top \left(\sum_{i=1}^\tau (V_{-k}^\top \widetilde{x}_i) (V_{-k}^\top \widetilde{x}_i)^\top + \tau\gamma V_{-k}^\top V_{-k}\right)^{-1} (V_{-k}^\top z) \\
    &\leq (V_k^\top z)^\top \left(\sum_{i=1}^\tau (V_k^\top \widetilde{x}_i) (V_k^\top \widetilde{x}_i)^\top + \tau\gamma I_k \right)^{-1} (V_k^\top z) + (V_{-k}^\top z)^\top \left( \tau\gamma V_{-k}^\top V_{-k}\right)^{-1} (V_{-k}^\top z) \\
    &= (V_k^\top z)^\top \left(\sum_{i=1}^\tau (V_k^\top \widetilde{x}_i) (V_k^\top \widetilde{x}_i)^\top + \tau\gamma I_k \right)^{-1} (V_k^\top z) + z^\top \left(\sum_{i=k+1}^d v_i v_i^\top \frac{1}{\tau\gamma}\right) z \\
    &\leq (V_k^\top z)^\top \left(\sum_{i=1}^\tau (V_k^\top \widetilde{x}_i) (V_k^\top \widetilde{x}_i)^\top + \tau\gamma I_k \right)^{-1} (V_k^\top z) + 2 z^\top \left(\sum_{i=k+1}^d v_i v_i^\top \frac{1}{\tau\gamma + \tau d_i}\right) z \\
    &\leq (1+\epsilon) (V_k^\top z)^\top \left(\tau \sum_{x \in \mc{X}} \lambda_x (V_k^\top x) (V_k^\top x)^\top + \tau\gamma I_k \right)^{-1} (V_k^\top z) + 2 z^\top \left(\sum_{i=k+1}^d v_i v_i^\top \frac{1}{\tau\gamma + \tau d_i}\right) z \\
    &\leq \max\{1+\epsilon,2\} z^\top \left( \tau \sum_{x \in \mc{X}} \lambda_x x x^\top + \tau\gamma I_d \right)^{-1} z.
\end{align*}
\end{proof}
Oftentimes $k$ can be much smaller than $\min\{d,|\mc{X}|\}$, especially for large $\gamma$. 
For instance, for $\mc{X}=\mc{Z}=\{ a \mb{e}_1 \} \cup \{ \mb{e}_i : i \in [d] \}$ with $a \gg \gamma = 1$, even as $d \rightarrow \infty$ we have that $k=1$ since $\lambda^*$ will be the majority of its mass on $\mb{e}_1$. 

\subsection{Connection to kernels}\label{kernel_exp_design}
We now get back to our initial setting. Consider $K$ the kernel matrix of $\mc{X} = \{x_1, \ldots, x_n\}$. Take $\widetilde{\Phi} \in \R^{n\times n}$ such that $K = \widetilde{\Phi} \widetilde{\Phi}^\top$ (can easily done by diagonalizing $K$). Consider the rows of $\widetilde{\Phi}$ and name these $\widetilde{\phi}(x_i)$. Then, we have by definition $\phi(x_i)^\top\phi(x_j) = [K]_{ij}$ and we have by construction $\widetilde{\phi}(x_i)^\top\widetilde{\phi}(x_j) = [K]_{ij}$, which importantly leads to $\phi(x_i)^\top\phi(x_j) = \widetilde{\phi}(x_i)^\top\widetilde{\phi}(x_j)$. \\
Fix $v\in\mc{V}\subset\mc{X}$. We have from Lemma~\ref{lmm:compute_kernel_with_alloc}
\begin{align*}
    \phi(v)^\top \left( \sum_{i=1}^\tau \phi(x_i) \phi(x_i)^\top + \tau\gamma I \right)^{-1}\phi(v) = \phi(v)^\top \phi(v) / (\tau\gamma) - \phi(v)^\top \Phi^\top ( \Phi \Phi^\top + \tau\gamma I_n)^{-1} \Phi \phi(v) / (\tau\gamma)  \;.
\end{align*}
This only involves scalar products of the form $\phi(x_i)^\top\phi(x_j)$, such that the property $\phi(x_i)^\top\phi(x_j) = \widetilde{\phi}(x_i)^\top\widetilde{\phi}(x_j)$ allows us to write the variance as
\begin{align*}
    \phi(v)^\top \left( \sum_{i=1}^\tau \phi(x_i) \phi(x_i)^\top + \tau\gamma I \right)^{-1}\phi(v)
    &= \phi(v)^\top \phi(v) / (\tau\gamma) - \phi(v)^\top \Phi^\top ( \Phi \Phi^\top + \tau\gamma I_n)^{-1} \Phi \phi(v) / (\tau\gamma) \\
    &= \widetilde{\phi}(v)^\top \widetilde{\phi}(v) / (\tau\gamma) - \widetilde{\phi}^\top \widetilde{\Phi}^\top ( \widetilde{\Phi} \widetilde{\Phi}^\top + \tau\gamma I_n)^{-1}\widetilde{\Phi} \widetilde{\phi}(v) / (\tau\gamma) \\
    &= \widetilde{\phi}(v)^\top \left( \sum_{i=1}^\tau \widetilde{\phi}(x_i) \widetilde{\phi}(x_i)^\top + \tau\gamma I_n \right)^{-1} \widetilde{\phi}(v)\;.
\end{align*}
The same trick allows us to write 
\begin{align*}
    \phi(v)^\top \left( \tau \sum_{x \in \mc{X}} \lambda_x \phi(x)\phi(x)^\top + \tau\gamma I \right)^{-1} \phi(v) = \widetilde{\phi}(v)^\top \left( \tau \sum_{x \in \mc{X}} \lambda_x \widetilde{\phi}(x)\widetilde{\phi}(x)^\top + \tau\gamma I_n \right)^{-1} \widetilde{\phi}(v) \;.
\end{align*}
Let $V \Delta V^\top$ be the eigenvalue decomposition of the matrix $\sum_{x \in \mc{X}} \lambda_x \widetilde{\phi}(x)\widetilde{\phi}(x)^\top$, and denote $\Delta=\text{diag}(d_1,\dots,d_n)$. We know from lemma \ref{finite_dim} that with $\tau = \Omega(\widetilde{d}(\lambda, \gamma)/\epsilon)$ and $\widetilde{d}(\lambda, \gamma) = \max\{ i : d_i \geq \gamma \}$ we can find an allocation $\{\widetilde{x}_i\}_{i=1}^\tau \subset \mc{X}$ such that 
\begin{align*}
    \widetilde{\phi}(v)^\top \left( \sum_{i=1}^\tau \widetilde{\phi}(\widetilde{x}_i) \widetilde{\phi}(\widetilde{x}_i) + \tau\gamma I_n \right)^{-1} \widetilde{\phi}(v) \leq \max\{1+\epsilon,2\} \widetilde{\phi}(v)^\top \left( \tau \sum_{x \in \mc{X}} \lambda_x \widetilde{\phi}(x)\widetilde{\phi}(x)^\top + \tau\gamma I_n \right)^{-1} \widetilde{\phi}(v) \;,
\end{align*}
which yields to the following result.\\
For any $v\in\mc{V}\subset\mc{X}$, as long as $\tau = \Omega(\widetilde{d}(\lambda, \gamma)/\epsilon)$, we can find an allocation $\{\widetilde{x}_i\}_{i=1}^\tau \subset \mc{X}$ such that 
\begin{align*}
    \phi(v)^\top \left( \sum_{i=1}^\tau \phi(\widetilde{x}_i) \phi(\widetilde{x}_i)^\top + \tau\gamma I_d \right)^{-1} \phi(v) \leq \max\{1+\epsilon,2\} \phi(v)^\top \left( \tau \sum_{x \in \mc{X}} \lambda_x \phi(x) \phi(x)^\top + \tau\gamma I_d \right)^{-1} \phi(v)\;
\end{align*}
and $\widetilde{d}(\lambda, \gamma) = \max\{ i : d_i \geq \gamma \}$.\\

And we can take the suppremum over $v\in\mc{V}$ to get to the result of Theorem~\ref{thm:rounding}.
\section{Main regret argument, Proof of Theorem~\ref{thm:regret_medofmean}}\label{sec:proof_thm_regret_medofmean}
In this section, we can consider without loss of generality that $\phi$ is the identity map. 
Indeed, the features of the actions - thus denoted $x$ here and $\phi(x)$ in the rest of the paper - appear in this proof only through scalar products. 

Define $f(\mc{V};\gamma) = \inf_{\lambda \in \triangle_{\mc{V}}} \max_{v \in \mc{V}} \| v \|_{(\sum_{y \in \mc{V}} \lambda_y y y^\top + \gamma I)^{-1}}^2$ and  $\bar{f}(\mc{X};\gamma) := \max_{\mc{V} \subseteq \mc{X}} f(\mc{V}; \gamma)$.

Define the event
\begin{align*}
    \mc{E} := \bigcap_{\ell=1}^\infty \bigcap_{x \in \mc{X}_\ell} \left\{  |\langle x, \widehat{\theta}_\ell - \theta_* \rangle| \leq  \epsilon_\ell +   (\sqrt{\gamma} \|\theta_* \|_2 + h) \sqrt{ \bar{f}(\mc{X};\gamma)} \right\}
\end{align*}
\begin{lemma}
We have $\P( \mc{E} ) \geq 1-\delta$.
\end{lemma}
\begin{proof}
For any $\mc{V} \subseteq \mc{X}$ and $x \in \mc{V}$ define
\begin{align*}
\mc{E}_{x,\ell}( \mc{V} ) = \{ |\langle x, \widehat{\theta}_\ell(\mc{V}) - \theta_* \rangle| \leq \epsilon_\ell +  (\sqrt{\gamma} \|\theta_* \|_2 + h) \sqrt{\bar{f}(\mc{X};\gamma)} \}
\end{align*}
where $\widehat{\theta}_\ell( \mc{V} )$ is the estimator that would be constructed by the algorithm at stage $\ell$ with $\mc{X}_\ell = \mc{V}$.
For fixed $\mc{V} \subset \mc{X}$ and $\ell \in \mathbb{N}$ we apply Theorem~\ref{thm:robust_estimator_lemma} with $\tau = \tau_\ell$  so that with probability at least $1-\frac{\delta}{2 \ell^2 |\mc{X}|}$ we have that for any $x \in \mc{V}$ 
\begin{align*}
    |\langle x, \widehat{\theta}_\ell(\mc{V}) - \theta_* \rangle| &\leq \| x \|_{A^{(\gamma)}(\lambda_\ell)^{-1}} \left( \sqrt{\gamma} \|\theta_* \|_2 + h + \HighProbConst\sqrt{(B^2 + \sigma^2) \frac{\log(4 \ell^2 |\mc{X}|/\delta) }{\tau_\ell}} \right) \\
    &\leq \sqrt{f(\mc{V};\gamma)} \left( \sqrt{\gamma} \|\theta_* \|_2 + h + \epsilon_\ell / \sqrt{f(\mc{V};\gamma)} \right) \\
    &\leq \epsilon_\ell +  (\sqrt{\gamma} \|\theta_* \|_2 + h)\sqrt{ \bar{f}(\mc{X};\gamma)}
\end{align*}
Noting that $ \mc{E} := \bigcap_{\ell=1}^\infty \bigcap_{x \in \mc{X}_\ell} \mc{E}_{x,\ell}( \mc{X}_\ell ) $ we have
\begin{align*}
\P\left( \bigcup_{\ell=1}^\infty \bigcup_{x \in \mc{X}_\ell} \{ \mc{E}^c_{x,\ell}( \mc{X}_\ell ) \} \right) &\leq \sum_{\ell=1}^\infty \P\left( \bigcup_{x \in \mc{X}_\ell} \{ \mc{E}^c_{x,\ell}( \mc{X}_\ell) \} \right) \\
&= \sum_{\ell=1}^\infty \sum_{\mc{V} \subseteq \mc{X}} \P\left( \bigcup_{x \in \mc{V}} \{ \mc{E}^c_{x,\ell}( \mc{V} ) \} , {\mc{X}}_\ell = \mc{V}\right) \\
&= \sum_{\ell=1}^\infty \sum_{\mc{V} \subseteq \mc{X}} \P\left(\bigcup_{x \in \mc{V}} \{ \mc{E}^c_{x,\ell}( \mc{V} ) \} \right) \P( {\mc{X}}_\ell = \mc{V}) \\
&\leq \sum_{\ell=1}^\infty  \sum_{\mc{V} \subseteq \mc{X}} \tfrac{ \delta}{2\ell^2 |\mc{X}|} |\mc{V}| \P( {\mc{X}}_\ell = \mc{V} ) \\
&\leq \sum_{\ell=1}^\infty  \sum_{\mc{V} \subseteq \mc{X}} \tfrac{ \delta}{2 \ell^2}  \P( \widehat{\mc{X}}_\ell= \mc{V})  \leq \delta
\end{align*}
\end{proof}

\begin{lemma}
    For all $\ell \in \mathbb{N}$ we have $\max_{x \in \mc{X}_{\ell}}\mu_* - \mu_x \leq \max\{ 16 \epsilon_\ell, 32(\sqrt{\gamma} \|\theta_* \|_2 + h)\sqrt{ \bar{f}(\mc{X};\gamma)} \}$.
\end{lemma}
\begin{proof}
An arm $x \in \mc{X}_\ell$ is discarded (i.e., not in $\mc{X}_{\ell+1}$) if $\max_{x' \in \mc{X}_\ell} \langle x', \widehat{\theta} \rangle - \langle x, \widehat{\theta} \rangle > 4 \epsilon_\ell$. 
Let $\bar{\ell} := \max\{ \ell : \epsilon_\ell >  2(\sqrt{\gamma} \|\theta_* \|_2 + h)\sqrt{ \bar{f}(\mc{X};\gamma)} \}$.
If $x_* = \arg\max_{x \in \mc{X}} \mu_x$ then $x_* \in \mc{X}_1$. Now if $x_* \in \mc{X}_\ell$ for some $\ell \leq \bar{\ell}$, then for any $x' \in \mc{X}_\ell$ we have
\begin{align*}
    \langle x', \widehat{\theta} \rangle - \langle x_*, \widehat{\theta} \rangle &\leq \langle x' - x_* ,\theta_* \rangle + 2\epsilon_\ell + 2 (\sqrt{\gamma} \|\theta_* \|_2 + h)\sqrt{ \bar{f}(\mc{X};\gamma)} \\
    &\leq \mu_x - \mu_{x_*} + 2h + 2\epsilon_\ell + 2 (\sqrt{\gamma} \|\theta_* \|_2 + h)\sqrt{ \bar{f}(\mc{X};\gamma)} \\
    &\leq2\epsilon_\ell + 4 (\sqrt{\gamma} \|\theta_* \|_2 + h)\sqrt{ \bar{f}(\mc{X};\gamma)} \\
    &\leq 4 \epsilon_\ell
\end{align*}
which implies $x_* \in \mc{X}_{\ell+1}$.
Moreover, suppose that $\ell \leq \bar{\ell}$ and there exists some $x \in \mc{X}_\ell$ such that $\mu_{*} - \mu_{x} > 8 \epsilon_\ell$, then
\begin{align*}
    \max_{x' \in \mc{X}_\ell} \langle x', \widehat{\theta} \rangle - \langle x, \widehat{\theta} \rangle &\geq \langle x_*, \widehat{\theta} \rangle - \langle x, \widehat{\theta} \rangle \\
    &\geq \langle x_* - x, {\theta}_* \rangle - 2\epsilon_\ell - 2 (\sqrt{\gamma} \|\theta_* \|_2 + h)\sqrt{ \bar{f}(\mc{X};\gamma)} \\
    &\geq \mu_* - \mu_x - 2h - 2\epsilon_\ell - 2 (\sqrt{\gamma} \|\theta_* \|_2 + h)\sqrt{ \bar{f}(\mc{X};\gamma)}\\
    &\geq \mu_* - \mu_x - 2\epsilon_\ell - 4(\sqrt{\gamma} \|\theta_* \|_2 + h)\sqrt{ \bar{f}(\mc{X};\gamma)} \\
    &\geq \mu_* - \mu_x - 4\epsilon_\ell  \\
    &> 4 \epsilon_\ell
\end{align*}
which implies $\max_{x \in \mc{X}_{\ell+1}} \mu_* - \mu_x \leq 8 \epsilon_\ell = 16 \epsilon_{\ell+1}$.
Because $\mc{X}_{\ell +1} \subseteq \mc{X}_\ell$ we have for $\ell > \bar{\ell}$ that
\begin{align*}
    \max_{x \in \mc{X}_{\ell}}\mu_* - \mu_x &\leq \max_{x \in \mc{X}_{\bar\ell+1}}\mu_* - \mu_x \\
    &\leq 16 \epsilon_{\bar\ell+1} \\
    &\leq 32 (\sqrt{\gamma} \|\theta_* \|_2 + h)\sqrt{ \bar{f}(\mc{X};\gamma)}.
\end{align*}
Thus, $\max_{x \in \mc{X}_{\ell}}\mu_* - \mu_x \leq \max\{ 16 \epsilon_\ell, 32 (\sqrt{\gamma} \|\theta_* \|_2 + h)\sqrt{ \bar{f}(\mc{X};\gamma)} \}$.
\end{proof}

We now compute the final regret bound. After $T$ steps of the algorithm, let $T_x$ denote the number of times arm $x$ is played.
Let $\Gamma = (\sqrt{\gamma} \|\theta_* \|_2 + h)\sqrt{ \bar{f}(\mc{X};\gamma)}$.
If $L$ is the final round reached after $T$ steps, we have
\begin{align*}
\sum_{x \in \mc{X}} (\mu_* - \mu_x) T_x 
&\leq \sum_{\ell=1}^L \max_{x \in \mc{X}_\ell} (\mu_* - \mu_x) \tau_\ell \\
&\leq \sum_{\ell=1}^L \tau_\ell \max\{ 16 \epsilon_\ell, 32 (\sqrt{\gamma} \|\theta_* \|_2 + h)\sqrt{ \bar{f}(\mc{X};\gamma)} \} \\
&\leq \sum_{\ell=1}^L \tau_\ell \max\{ 16 \epsilon_\ell, 32 \Gamma \} \\
&\leq \sum_{\ell: \epsilon_\ell < 2 \Gamma } 32 \Gamma \tau_\ell + \sum_{\ell: \epsilon_\ell \geq 2 \Gamma }   \epsilon_\ell  \tau_\ell \\
&\leq \sum_{\ell: \epsilon_\ell < 2 \Gamma } 32 \Gamma \tau_\ell + 16 \nu T + \sum_{\ell: \epsilon_\ell \geq 2 \Gamma \vee  \nu }  16 \epsilon_\ell  \tau_\ell \\
&\leq \sum_{\ell: \epsilon_\ell < 2 \Gamma } 32 \Gamma \tau_\ell + 16 \nu T + \sum_{\ell: \epsilon_\ell \geq \nu } 16  \epsilon_\ell  \tau_\ell \\
&\leq c\left( \Gamma  T + \nu T + \sum_{\ell: \epsilon_\ell \geq  \nu } \epsilon_\ell \cdot \left(\HighProbConst (B^2 + \sigma^2) \epsilon_\ell^{-2} f( \mc{X}_\ell ; \gamma) \log(4 \ell^2 |\mc{X}|/\delta)+c_1\log(|\mc{X}|/\delta)\right)\right) \\
&\leq c\left( \Gamma  T + \nu T + \cdot \left(\HighProbConst (B^2 + \sigma^2) \epsilon_\ell^{-2} f( \mc{X}_\ell ; \gamma) \log(4 \ell^2 |\mc{X}|/\delta)+c_1\log(|\mc{X}|/\delta)\right)
\sum_{\ell: \epsilon_\ell \geq \nu }  \epsilon_\ell^{-1}\right) \\
&\leq c\left(\Gamma  T +  \nu T + \nu^{-1} \HighProbConst   (B^2 + \sigma^2) \bar{f}( \mc{X} ; \gamma) \log(4 \lceil \log_2(1/\nu) \rceil^2 |\mc{X}|/\delta) +c_1\log(|\mc{X}|/\delta)\right).
\end{align*}
Choosing $\nu = \sqrt{ \HighProbConst   (B^2 + \sigma^2) \bar{f}( \mc{X} ; \gamma) \log(  |\mc{X}|/\delta) / T}$ and plugging $\Gamma$ back in yields
\begin{align*}
\sum_{x \in \mc{X}} (\mu_* - \mu_x) T_x &\leq c' \sqrt{\bar{f}( \mc{X} ; \gamma)} \left( T (\sqrt{\gamma} \|\theta_* \|_2 + h) + \sqrt{(B^2 + \sigma^2) \log(  |\mc{X}|\log(T)/\delta) T} \right)+c_1\log(|\mc{X}|/\delta).
\end{align*}
Choosing $\gamma = 1/T$ yields
\begin{align*}
\sum_{x \in \mc{X}} (\mu_* - \mu_x) T_x &\leq c' \sqrt{\bar{f}( \mc{X} ; 1/T)} \left( h T + \sqrt{(\|\theta_*\|^2 + \max_{x \in \mc{X}}\langle x,\theta_* \rangle + \sigma^2) \log(  |\mc{X}|\log(T)/\delta) T} \right) +c_1\log(|\mc{X}|/\delta).
\end{align*}

\section{Main robust pure exploration result, Proof of Theorem~\ref{thm:samplecomplexity_medofmean}}

For any $\mc{V} \subset \mc{Z}$ define $f(\mc{X},\mc{V};\gamma) = \min_{\lambda\in\triangle_\mc{X
}} \max_{v,v' \in \mc{V}} \| v - v' \|_{( \sum_{x \in \mc{X}} \lambda_x \phi(x)\phi(x)^\top + \gamma I)^{-1}}^2$

\begin{lemma}
Define
\begin{align*}
    \bar{\epsilon} =  8 \min\{ \epsilon \geq 0 : 4(\sqrt{\gamma} \|\theta_* \|_2 + h)(2+ \sqrt{f(\mc{X},\{ z \in \mc{Z} : \langle \phi(z_*)- \phi(z), \theta_* \rangle \leq \epsilon \};\gamma)}) \leq \epsilon \}.
\end{align*}
Then $\max_{z \in \mc{Z}_{\ell}}\mu_* - \mu_z \leq 8 \max\{ \epsilon_\ell, \bar\epsilon \}$ for all $\ell \geq 0$ with probability at least $1-\delta$.
\end{lemma}

We use Theorem~\ref{thm:robust_estimator_lemma} again, with here $\mc{V} \subset \mc{Z} \subset \R^d$:
\begin{align*}
    \max_{v \in \mc{V}} \frac{ | \langle \theta_*, \phi(v) \rangle - W^{(\phi(v))}| }{\|\phi(v)\|_{(\sum_{x \in \mc{X}} \lambda_x \phi(x)\phi(x)^\top + \gamma I)^{-1}} } \leq  \sqrt{\gamma} \|\theta_* \|_2 + h +   \HighProbConst \sqrt{\tfrac{(B^2 + \sigma^2)}{T}\log(2|\mc{V}|/\delta) } 
\end{align*}
which motivates the choice
\begin{align*}
    \tau_\ell =\HighProbConst^2 \epsilon_\ell^{-2} f(\mc{X},\mc{Z}_\ell;\gamma) (B^2 + \sigma^2) \log(2\ell^2|\mc{Z}|^2/\delta)
\end{align*}

Define the event
\begin{align*}
    \mc{E} := \bigcap_{\ell=1}^\infty \bigcap_{z,z' \in \mc{Z}_\ell} \left\{  |\langle \phi(z) - \phi(z'), \widehat{\theta}_\ell - \theta_* \rangle| \leq  \epsilon_\ell +   (\sqrt{\gamma} \|\theta_* \|_2 + h) \sqrt{ f(\mc{X},\mc{Z}_\ell;\gamma)} \right\}
\end{align*}
\begin{lemma}
We have $\P( \mc{E} ) \geq 1-\delta$.
\end{lemma}
\begin{proof}
This proof follows the analogous result for regret almost identically. We include it for completeness.
For any $\mc{V} \subseteq \mc{Z}$ and $x \in \mc{X}$ define
\begin{align*}
\mc{E}_{z,z',\ell}( \mc{V} ) = \{ |\langle \phi(z)-\phi(z'), \widehat{\theta}_\ell(\mc{V}) - \theta_* \rangle|  \leq \epsilon_\ell +  (\sqrt{\gamma} \|\theta_* \|_2 + h) \sqrt{f(\mc{X},\mc{Z}_\ell;\gamma)} \}
\end{align*}
where $\widehat{\theta}_\ell( \mc{V} )$ is the estimator that would be constructed by the algorithm at stage $\ell$ with $\mc{Z}_\ell = \mc{V}$.
For fixed $\mc{V} \subset \mc{Z}$ and $\ell \in \mathbb{N}$ we apply Theorem~\ref{thm:robust_estimator_lemma} with $T = \tau_\ell$  so that with probability at least $1-\frac{\delta}{ \ell^2 |\mc{Z}|^2}$ we have that for any $z, z' \in \mc{V}$
\begin{align*}
    |\langle \phi(z)-\phi(z'), \widehat{\theta}_\ell(\mc{V}) - \theta_* \rangle| &\leq \| \phi(z)-\phi(z') \|_{A^{(\gamma)}(\lambda_\ell)^{-1}} ( \sqrt{\gamma} \|\theta_* \|_2 + h + \HighProbConst\sqrt{(B^2 + \sigma^2) \frac{\log(2 \ell^2 |\mc{Z}|^2/\delta) }{\tau_\ell}} ) \\
    &\leq \sqrt{f(\mc{X},\mc{V};\gamma)} \left( \sqrt{\gamma} \|\theta_* \|_2 + h + \epsilon_\ell / \sqrt{f(\mc{X},\mc{V};\gamma)} \right) \\
    &\leq \epsilon_\ell +  (\sqrt{\gamma} \|\theta_* \|_2 + h)\sqrt{ f(\mc{X},\mc{V};\gamma)}
\end{align*}
Noting that $ \mc{E} := \bigcap_{\ell=1}^\infty \bigcap_{z,z' \in \mc{Z}_\ell} \mc{E}_{z,z',\ell}( \mc{Z}_\ell )  $ we have
\begin{align*}
\P\left( \bigcup_{\ell=1}^\infty \bigcup_{z,z' \in \mc{Z}_\ell} \{ \mc{E}^c_{z,z',\ell}( \mc{Z}_\ell ) \} \right) &\leq \sum_{\ell=1}^\infty \P\left( \bigcup_{z,z' \in \mc{Z}_\ell} \{ \mc{E}^c_{z,z',\ell}( \mc{Z}_\ell) \} \right) \\
&= \sum_{\ell=1}^\infty \sum_{\mc{V} \subseteq \mc{Z}} \P\left( \bigcup_{z,z' \in \mc{V}} \{ \mc{E}^c_{z,z',\ell}( \mc{V} ) \} , {\mc{Z}}_\ell = \mc{V}\right) \\
&= \sum_{\ell=1}^\infty \sum_{\mc{V} \subseteq \mc{Z}} \P\left(\bigcup_{z,z' \in \mc{V}} \{ \mc{E}^c_{z,z',\ell}( \mc{V} ) \} \right) \P( {\mc{Z}}_\ell = \mc{V}) \\
&\leq \sum_{\ell=1}^\infty  \sum_{\mc{V} \subseteq \mc{Z}} \tfrac{ \delta}{\ell^2 |\mc{Z}|^2} \binom{|\mc{V}|}{2} \P( {\mc{Z}}_\ell = \mc{V} ) \\
&\leq \sum_{\ell=1}^\infty  \sum_{\mc{V} \subseteq \mc{Z}} \tfrac{ \delta}{2 \ell^2}  \P( \mc{Z}_\ell= \mc{V})  \leq \delta
\end{align*}
\end{proof}

\begin{lemma}
Define $S_1 = \mc{Z}$ and $S_{\ell+1} = \{z \in S_\ell : \langle \phi(z_*)- \phi(z), \theta_* \rangle \leq 3\epsilon_\ell + (\sqrt{\gamma} \|\theta_* \|_2 + h) \sqrt{f(\mc{X},S_\ell;\gamma)} \} $.
Define 
\begin{align*}
    \bar\ell &= \max \{ \ell : (\sqrt{\gamma} \|\theta_* \|_2 + h)(2+ \sqrt{f(\mc{X},S_\ell;\gamma)}) \leq \epsilon_\ell \} .
\end{align*}
    For all $\ell \in \mathbb{N}$ we have $\max_{z \in \mc{Z}_{\ell}}\mu_* - \mu_z \leq 8 \max\{ \epsilon_\ell, \epsilon_{\bar\ell} \}$.
\end{lemma}
\begin{proof}
An arm $z \in \mc{Z}_\ell$ is discarded (i.e., not in $\mc{Z}_{\ell+1}$) if $\max_{z' \in \mc{Z}_\ell} \langle \phi(z')-\phi(z), \widehat{\theta}_\ell \rangle > 2 \epsilon_\ell$.

We will show $\{z_* \in \mc{Z}_\ell \} \cap \{ \mc{Z}_\ell \subset S_\ell \} \cap \{ \ell \leq \bar\ell\} \implies \{z_* \in \mc{Z}_{\ell+1} \} \cap \{ \mc{Z}_{\ell+1} \subset S_{\ell+1} \}$.
Noting that $\{z_* \in \mc{Z}_\ell \} \cap \{ \mc{Z}_\ell \subset S_\ell \}$ holds for $\ell=1$, we will assume an inductive hypothesis  of this condition for some $\ell \leq \bar\ell$.

First we will show $\{z_* \in \mc{Z}_\ell \} \cap \{ \mc{Z}_\ell \subset S_\ell \} \cap \{ \ell \leq \bar\ell\} \implies \{z_* \in \mc{Z}_{\ell+1} \}$.
On $\{z_* \in \mc{Z}_\ell \} \cap \{ \mc{Z}_\ell \subset S_\ell \} \cap \{ \ell \leq \bar\ell\}$, we have for any $z' \in \mc{Z}_\ell$ that
\begin{align*}
    \langle \phi(z')-\phi(z_*), \widehat{\theta}_\ell \rangle &\leq \langle \phi(z')-\phi(z_*), \theta_* \rangle + \epsilon_\ell +  (\sqrt{\gamma} \|\theta_* \|_2 + h) \sqrt{f(\mc{X},\mc{Z}_\ell;\gamma)} \\
    &\leq \mu_{z'} - \mu_{z_*} + 2h + \epsilon_\ell +  (\sqrt{\gamma} \|\theta_* \|_2 + h) \sqrt{f(\mc{X},\mc{Z}_\ell;\gamma)} \\
    &\leq\epsilon_\ell + (\sqrt{\gamma} \|\theta_* \|_2 + h)(2+ 
    \sqrt{f(\mc{X},\mc{Z}_\ell;\gamma)})  \\
    &\leq\epsilon_\ell + (\sqrt{\gamma} \|\theta_* \|_2 + h)(2+ \sqrt{f(\mc{X},S_\ell;\gamma)}) \\
    &\leq 2 \epsilon_\ell
\end{align*}
which implies $z_*$ is not eliminated, that is, $z_* \in \mc{Z}_{\ell+1}$.
The second-to-last inequality follows from
\begin{align*}
    f(\mc{X},\mc{Z}_\ell;\gamma) &= \inf_{\lambda\in\triangle_{\mc{X}}} \max_{z,z' \in \mc{Z}_\ell} \| \phi(z) - \phi(z') \|_{(\sum_{x \in \mc{X}} \lambda_x \phi(x)\phi(x)^\top + \gamma I)^{-1}}^2 \\
    &\leq \inf_{\lambda\in\triangle_{\mc{X}}} \max_{z,z' \in S_\ell} \| \phi(z) - \phi(z') \|_{(\sum_{x \in \mc{X}} \lambda_x \phi(x)\phi(x)^\top + \gamma I)^{-1}}^2 \\
    &= f(\mc{X},S_\ell;\gamma).
\end{align*}

Now we will show $\{z_* \in \mc{Z}_\ell \} \cap \{ \mc{Z}_\ell \subset S_\ell \} \cap \{ \ell \leq \bar\ell\} \implies \{ \mc{Z}_{\ell+1} \subset S_{\ell+1} \}$.
For any $z \in \mc{Z}_\ell \cap S_{\ell+1}^c$ we have
\begin{align*}
    \max_{z' \in \mc{Z}_\ell}\langle \phi(z')-\phi(z), \widehat{\theta}_\ell \rangle &\geq \langle \phi(z_*)-\phi(z), \widehat{\theta}_\ell \rangle \\
    &\geq \langle \phi(z_*)-\phi(z), \theta_* \rangle - \epsilon_\ell -  (\sqrt{\gamma} \|\theta_* \|_2 + h) \sqrt{f(\mc{X},\mc{Z}_\ell;\gamma)}\\
    &> 3\epsilon_\ell + (\sqrt{\gamma} \|\theta_* \|_2 + h) \sqrt{f(\mc{X},S_\ell;\gamma)}  - \epsilon_\ell -  (\sqrt{\gamma} \|\theta_* \|_2 + h) \sqrt{f(\mc{X},\mc{Z}_\ell;\gamma)} \\
    &\geq 2\epsilon_\ell
\end{align*}
which implies $z \not\in \mc{Z}_{\ell+1}$, and $\mc{Z}_{\ell+1} \subset S_{\ell+1}$.

Thus, for $\ell \leq \bar\ell$ we have
\begin{align*}
    \max_{z \in \mc{Z}_{\ell}}\mu_* - \mu_z &\leq \max_{z \in \mc{Z}_{\ell}}  \langle \phi(z_*)-\phi(z), \theta_* \rangle + 2h \\
    &\leq \max_{z \in \mc{S}_{\ell}}  \langle \phi(z_*)-\phi(z), \theta_* \rangle + 2h \\
    &\leq 3\epsilon_{\ell-1} + 2h +  (\sqrt{\gamma} \|\theta_* \|_2 + h) \sqrt{f(\mc{X},S_{\ell-1};\gamma)} \\
    &\leq 3\epsilon_{\ell-1} + (\sqrt{\gamma} \|\theta_* \|_2 + h)(2 + \sqrt{f(\mc{X},S_{\ell-1};\gamma)}) \\
    &\leq 4 \epsilon_{\ell-1} = 8 \epsilon_\ell.
\end{align*}
And because $\mc{Z}_{\ell+1} \subseteq \mc{Z}_\ell$ we always have that $\max_{z \in \mc{Z}_{\ell}}\mu_* - \mu_z \leq 8 \max\{ \epsilon_\ell, \epsilon_{\bar\ell} \}$. Note that 
\begin{align*}
    \bar\ell &= \max \{ \ell : (\sqrt{\gamma} \|\theta_* \|_2 + h)(2+ \sqrt{f(\mc{X},S_\ell;\gamma)}) \leq \epsilon_\ell \} \\
    &\geq \max \{ \ell : (\sqrt{\gamma} \|\theta_* \|_2 + h)(2+ \sqrt{f(\mc{X},\{ z \in \mc{Z} : \langle \phi(z_*)- \phi(z), \theta_* \rangle \leq 4\epsilon_\ell \};\gamma)}) \leq \epsilon_\ell \} \\
    &= \max \{ \ell : 4(\sqrt{\gamma} \|\theta_* \|_2 + h)(2+ \sqrt{f(\mc{X},\{ z \in \mc{Z} : \langle \phi(z_*)- \phi(z), \theta_* \rangle \leq 4\epsilon_\ell \};\gamma)}) \leq 4 \epsilon_\ell \} \\
    &= -2 + \max \{ \ell : 4(\sqrt{\gamma} \|\theta_* \|_2 + h)(2+ \sqrt{f(\mc{X},\{ z \in \mc{Z} : \langle \phi(z_*)- \phi(z), \theta_* \rangle \leq \epsilon_\ell \};\gamma)}) \leq  \epsilon_\ell \} \\
    &\geq -3 - \log_2( \min\{ \epsilon > 0 :  4(\sqrt{\gamma} \|\theta_* \|_2 + h)(2+ \sqrt{f(\mc{X},\{ z \in \mc{Z} : \langle \phi(z_*)- \phi(z), \theta_* \rangle \leq \epsilon \};\gamma)}) \leq \epsilon \})
\end{align*}
which defines $\bar{\epsilon}$.
\end{proof}

The sample complexity to return an $8(\Delta \vee \bar\epsilon)$-good arm is equal to
\begin{align*}
    &\sum_{\ell=1}^{\lceil \log_2(8 (\Delta \vee \bar\epsilon )^{-1} )\rceil}\tau_\ell = \sum_{\ell=1}^{\lceil \log_2(8 (\Delta \vee \bar\epsilon )^{-1} )\rceil}\left\lceil\max\left\{c_1\log(|\mc{Z}|/\delta),  \HighProbConst \epsilon_\ell^{-2} f(\mc{X},\mc{Z}_\ell;\gamma) (B^2 + \sigma^2) \log(2\ell^2|\mc{Z}|^2/\delta)\right\}\right\rceil \\
    &\leq c\left(c_1\log(|\mc{Z}|/\delta) \log_2(8 (\Delta \vee \bar\epsilon )^{-1} ) + \HighProbConst (B^2 + \sigma^2) \log(2\lceil \log_2(8 (\Delta \vee \bar\epsilon )^{-1} )\rceil^2|\mc{Z}|^2/\delta) \sum_{\ell=1}^{\lceil \log_2(8 (\Delta \vee \bar\epsilon )^{-1} )\rceil} \epsilon_\ell^{-2} f(\mc{X},\mc{Z}_\ell;\gamma) \right) \\
    &\leq c\left(c_1\log(|\mc{Z}|/\delta) \log_2(8 (\Delta \vee \bar\epsilon )^{-1} ) + \HighProbConst (B^2 + \sigma^2) \log(2\lceil \log_2(8 (\Delta \vee \bar\epsilon )^{-1} )\rceil^2|\mc{Z}|^2/\delta) \sum_{\ell=1}^{\lceil \log_2(8 (\Delta \vee \bar\epsilon )^{-1} )\rceil} \epsilon_\ell^{-2} f(\mc{X},S_\ell;\gamma)\right) \\
    &\leq c\left(c_1\log(|\mc{Z}|/\delta) \log_2(8 (\Delta \vee \bar\epsilon )^{-1} ) + 16\lceil \log_2(8 (\Delta \vee \bar\epsilon )^{-1} )\rceil \HighProbConst (B^2 + \sigma^2) \log(2\lceil \log_2(8 (\Delta \vee \bar\epsilon )^{-1} )\rceil^2|\mc{Z}|^2/\delta) \rho^*(\gamma,\bar\epsilon)\right)
\end{align*}
where the last line follows from
\begin{align*}
    \rho^*(\gamma,\bar\epsilon) &= \inf_{\lambda \in \Delta_{\mathcal{X}}}\sup_{z\in \mathcal{Z}}\frac{\|\phi(z_*) - \phi(z)\|^2_{\left(\sum_{x\in \mathcal{X}}\lambda_x \phi(x)\phi(x)^\top + \gamma I\right)^{-1}}}{\max\{\bar\epsilon^2,\langle\phi(z_*) - \phi(z), \theta_*\rangle^2)\}} \\
    &= \inf_{\lambda \in \Delta_{\mathcal{X}}}\sup_{\ell\leq \lceil \log_2(8 (\Delta \vee \bar\epsilon )^{-1} )\rceil}\sup_{z\in S_\ell}\frac{\|\phi(z_*) - \phi(z)\|^2_{\left(\sum_{x\in \mathcal{X}}\lambda_x \phi(x)\phi(x)^\top + \gamma I\right)^{-1}}}{\max\{\bar\epsilon^2,((\phi(z_*) - \phi(z))^\top \theta_*)^2\}} \\
    &\geq \inf_{\lambda \in \Delta_{\mathcal{X}}}\frac{1}{\lceil \log_2(8 (\Delta \vee \bar\epsilon )^{-1} )\rceil}\sum_{\ell = 1}^{\lceil \log_2(8 (\Delta \vee \bar\epsilon )^{-1} )\rceil}\sup_{z\in S_\ell}\frac{\|\phi(z_*) - \phi(z)\|^2_{\left(\sum_{x\in \mathcal{X}}\lambda_x \phi(x)\phi(x)^\top + \gamma I\right)^{-1}}}{\max\{\bar\epsilon^2,((\phi(z_*) - \phi(z))^\top \theta_*)^2\}} \\
    &\geq \frac{1}{\lceil \log_2(8 (\Delta \vee \bar\epsilon )^{-1} )\rceil}\sum_{\ell = 1}^{\lceil \log_2(8 (\Delta \vee \bar\epsilon )^{-1} )\rceil}\inf_{\lambda \in \Delta_{\mathcal{X}}}\sup_{z\in S_\ell}\frac{\|\phi(z_*) - \phi(z)\|^2_{\left(\sum_{x\in \mathcal{X}}\lambda_x \phi(x)\phi(x)^\top + \gamma I\right)^{-1}}}{\max\{\bar\epsilon^2,((\phi(z_*) - \phi(z))^\top \theta_*)^2\}} \\
    &\geq \frac{1}{4\lceil \log_2(8 (\Delta \vee \bar\epsilon )^{-1} )\rceil}\sum_{\ell = 1}^{\lceil \log_2(8 (\Delta \vee \bar\epsilon )^{-1} )\rceil}2^{2\ell}\inf_{\lambda \in \Delta_{\mathcal{X}}}\sup_{z\in S_\ell}\|\phi(z_*) - \phi(z)\|^2_{\left(\sum_{x\in \mathcal{X}}\lambda_x \phi(x)\phi(x)^\top + \gamma I\right)^{-1}} \\
    &\geq \frac{1}{16\lceil \log_2(8 (\Delta \vee \bar\epsilon )^{-1} )\rceil}\sum_{\ell = 1}^{\lceil \log_2(8 (\Delta \vee \bar\epsilon )^{-1} )\rceil}2^{2\ell}\inf_{\lambda \in \Delta_{\mathcal{X}}}\sup_{z, z'\in S_\ell}\|\phi(z) - \phi(z')\|^2_{\left(\sum_{x\in \mathcal{X}}\lambda_x \phi(x)\phi(x)^\top + \gamma I\right)^{-1}} \\
    &= \frac{1}{16\lceil \log_2(8 (\Delta \vee \bar\epsilon )^{-1} )\rceil}\sum_{\ell = 1}^{\lceil \log_2(8 (\Delta \vee \bar\epsilon )^{-1} )\rceil}2^{2\ell}f(\mc{X},S_\ell;\gamma)  \;.
\end{align*}

\section{Proofs for the regret bound and the sample complexity of the alternative baseline}
Note importantly that in this section the stochastic noise is sub-gaussian.

\subsection{Concentration of the sparse estimator}\label{section:proof_rls}
\begin{lemma}\label{lem:rls_estimator_lemma}
Fix a finite set $\mc{V} \subset \mc{H}$, finite set $\mc{X}$, and let $\phi:\mc{X} \rightarrow \mc{H}$. Fix any $x_1, \ldots, x_T$ and assume $y_t = \langle \phi(x_t), \theta_* \rangle + \xi_t + \eta_t$ where each $\xi_t$ is independent, mean-zero, sub-gaussian with parameter $\sigma^2$, and each $\eta_t$ satisfies $|\eta_t| \leq h$. If $\widehat{\theta}$ is the regularized least squares estimator with regularization $\gamma>0$ then
\begin{align*}
    \max_{v \in \mc{V}} \frac{| \langle \widehat{\theta}, v \rangle - \langle \theta_*, v \rangle|}{\|v\|_{(\sum_{i=1}^T \phi(x_i)\phi(x_i)^\top + \gamma I)^{-1}}} &\leq (\sqrt{\gamma} \|\theta_* \|_2 + h\sqrt{T} + \sqrt{2\sigma^2\log(2|\mc{V}|/\delta) } )
\end{align*}
with probability at least $1-\delta$.
\end{lemma}
\begin{proof}
Recall first the definition of regularized least squares estimator $\widehat{\theta}$: 
\begin{align*}
    \widehat{\theta} = G_\gamma^{-1}\Phi^\top Y
\end{align*}
where $G_\gamma := \sum_{i=1}^T \phi(x_i)\phi(x_i)^\top + \gamma I$. Defining $\xi := (\xi_1, \ldots, \xi_T)$ and $\eta := (\eta_1, \ldots, \eta_T)$, holds
\begin{align*}
    v^\top (\widehat{\theta} - \theta_*)
    &= v^\top G_\gamma^{-1} \Phi^\top Y - v^\top \theta_*\\
    &= v^\top G_\gamma^{-1} \Phi^\top (\Phi \theta_* + \xi + \eta) - v^\top \theta_*\\
    &= v^\top G_\gamma^{-1} \Phi^\top \Phi \theta_* - v^\top \theta_* + v^\top G_\gamma^{-1} \Phi^\top \xi + v^\top G_\gamma^{-1} \Phi^\top \eta\;.
\end{align*}
We study each term separately:
\begin{align*}
    |v^\top G_\gamma^{-1} \Phi^\top \eta| 
    &\leq \|\eta\|_\infty \|\Phi G_\gamma^{-1} v\|_1\\
    &\leq h \sqrt{T} \|\Phi G_\gamma^{-1} v\|_2\\
    &\leq h \sqrt{T} \|v\|_{G_\gamma^{-1}}\;,
\end{align*}
with probability at least $1-\delta$ 
\begin{align*}
    |v^\top G_\gamma^{-1} \Phi^\top \xi| 
    &\leq \|v\|_{G_\gamma^{-1}} \sqrt{2\sigma^2 \log\left(\frac{2}{\delta}\right)}\;,
\end{align*}
and
\begin{align*}
    v^\top G_\gamma^{-1} \Phi^\top \Phi \theta_* - v^\top \theta_*
    &= - \gamma v^\top (\Phi^\top \Phi + \gamma I )^{-1} \theta_* \;
\end{align*}
is bounded using Cauchy-Schwarz inequality:
\begin{align*}
    &|\gamma v (\Phi^\top \Phi + \gamma I )^{-1}\theta_*| \\
    &\leq \gamma \|\theta_*\| \sqrt{v^\top (\Phi^\top \Phi + \gamma I )^{-1}I(\Phi^\top \Phi + \gamma I )^{-1}v}\\
    &= \gamma \|\theta_*\| \sqrt{v^\top (\Phi^\top \Phi + \gamma I )^{-1}\gamma^{-1} \gamma I(\Phi^\top \Phi + \gamma I )^{-1}v}\\
    &\leq \gamma^{1/2} \|\theta_*\| \sqrt{v^\top (\Phi^\top \Phi + \gamma I )^{-1}(\gamma I +\Phi^\top \Phi) (\Phi^\top \Phi + \gamma I )^{-1}v}\\
    &= \gamma^{1/2} \|\theta_*\| \|v\|_{\left(\Phi^\top \Phi + \gamma I\right)^{-1}}\\
    &= \gamma^{1/2} \|\theta_*\| \|v\|_{G_\gamma^{-1}}\;.
\end{align*}
So 
\begin{align*}
    |v^\top (\widehat{\theta} - \theta_*)| 
    &\leq \sqrt{\gamma} \|\theta_*\|_2 \|v\|_{G_\gamma^{-1}} + \|v\|_{G_\gamma^{-1}} \sqrt{2\sigma^2 \log\left(\frac{2}{\delta}\right)} + h \sqrt{n} \|v\|_{G_\gamma^{-1}} \\
    &= \|v\|_{G_\gamma^{-1}}\left(\sqrt{\gamma} \|\theta_*\|_2 + h \sqrt{T} + \sqrt{2\sigma^2 \log\left(\frac{2}{\delta}\right)}\right).
\end{align*}
Union bounding over all $v\in\mc{V}$ completes the proof.
\end{proof}

\subsection{Regret bound}
For the same reason as in section~\ref{sec:proof_thm_regret_medofmean}, we can consider without loss of generality that $\phi$ is the identity map in this section. 
Indeed, the features of the actions - thus denoted $x$ here and $\phi(x)$ in the rest of the paper - appear in this proof only through scalar products.
\begin{algorithm}[tb]
  \caption{PTR for Regret minimization}
  \label{alg:regret_rls}
\begin{algorithmic}
\State {\bfseries Input:} Finite sets $\mc{X} \subset \R^d$ ($|\mc{X}| = n$), feature map $\phi$, confidence level $\delta \in (0, 1)$, regularization $\gamma$, sub-gaussian parameter $\sigma$. \\
\State Set $\mc{X}_1 \gets \mc{X}, \ell \gets 1$\\
\While{$|\mc{X}_\ell|>1$}
    \State Let $\lambda_\ell\in \triangle_{\mc{X}_\ell}$ be a minimizer of $f(\lambda, \mc{X}_\ell, \gamma)$ where
        \begin{align*}
            f( \mc{V}; \gamma) = \inf_{\lambda\in\triangle_\mc{V}}f(\lambda, \mc{V}, \gamma) = \inf_{\lambda\in\triangle_\mc{V}}
            \max_{v \in \mc{V}} \| \phi(v) \|_{(\sum_{y \in \mc{V}} \lambda_y \phi(y) \phi(y)^\top + \gamma I)^{-1}}^2
        \end{align*}
    \State Set $\epsilon_\ell = 2^{-\ell}$ and $\tau_\ell := \lceil \max\{2\sigma^2 \epsilon_\ell^{-2} f( \mc{X}_\ell ; \gamma) \log(4 \ell^2 |\mc{X}|/\delta), \widetilde{d}(\gamma, \lambda_\ell)\}\rceil$ 
    \State Use the PTR procedure of section~\ref{sec:RKHS_rounding} to find sparse allocation $\{\widetilde{x}_i\}_{i=1}^{\tau_\ell} \subset \mc{X}_\ell$ from $\lambda_\ell$.
    \State Take each action $x \in \{\widetilde{x}_i\}_{i=1}^{\tau_\ell}$ with corresponding features $\Phi$ and rewards $Y$\\
    \State Compute $\widehat{\theta}_{\ell} = (\Phi^\top \Phi + \tau_\ell\gamma I)^{-1}\Phi^\top Y$
    \State Update active set:
        \begin{align*}
        \displaystyle 
        \mc{X}_{\ell+1} = \Big\{x \in \mc{X}_\ell, &\max_{x' \in \mc{X}_\ell} \langle \phi(x') - \phi(x), \widehat{\theta}_{\ell} \rangle < 8\epsilon_\ell \Big\}
        \end{align*}
    \State $\ell \gets \ell + 1$ \\
    \EndWhile
    \State Play unique element of $       \mc{X}_\ell$ indefinitely.
\end{algorithmic}
\end{algorithm}
\begin{theorem}\label{thm:regret_rls}
With probability at least $1-\delta$, the regret of Algorithm~\ref{alg:regret_rls} satisfies
\begin{align*}
    \sum_{t=1}^T \mu_x - \mu_{x_t} \leq c\left(\widetilde{d}(\gamma) + \sqrt{\max_{\mc{V}\subset \mc{X}}f(\mc{V},\gamma)} \left( T (\sqrt{\gamma} \|\theta_* \|_2 + h) + \sqrt{\left(\sigma^2 \log(  |\mc{X}|\log(T)/\delta) \right) T} \right)\right)
\end{align*}
where $\displaystyle f(\mc{V},\gamma) = \inf_{\lambda \in \triangle_\mc{V}} \sup_{y \in \mc{V}} \|y\|^2_{( \sum_{x\in\mc{X}} \lambda_x x x^\top + \gamma I )^{-1}}$ and $\widetilde{d}(\gamma) = \max_{\ell \leq L}\widetilde{d}(\gamma, \lambda_\ell) \leq \max_{\mc{V}\subset \mc{X}}\sup_{\lambda\in\triangle_\mc{V}}\widetilde{d}(\gamma, \lambda)$.
\end{theorem}

Recall the definition of $f(\mc{V};\gamma) = \inf_{\lambda \in \triangle_{\mc{V}}} \max_{v \in \mc{V}} \| v \|_{(\sum_{y \in \mc{V}} \lambda_y y y^\top + \gamma I)^{-1}}^2$ and  $\bar{f}(\mc{X};\gamma) := \max_{\mc{V} \subseteq \mc{X}} f(\mc{V}; \gamma)$.
Define the event
\begin{align*}
    \mc{E} := \bigcap_{\ell=1}^\infty \bigcap_{x \in \mc{X}_\ell} \left\{  |\langle x, \widehat{\theta}_\ell - \theta_* \rangle| \leq 2\epsilon_\ell + 2\left(\sqrt{\gamma} \|\theta_*\|_2 + h\right)\sqrt{\bar{f}(\mc{X};\gamma)} \right\}
\end{align*}

\begin{lemma}
We have $\P( \mc{E} ) \geq 1-\delta$.
\end{lemma}
\begin{proof}
For any $\mc{V} \subseteq \mc{X}$ and $x \in \mc{V}$ define 
\begin{align*}
    \mathcal{E}_{\ell, x} = \left\{|x^\top (\widehat{\theta}_\ell( \mc{V} ) - \theta_*)| \leq 2\epsilon_\ell + 2\left(\sqrt{\gamma} \|\theta_*\|_2 + h\right)\sqrt{\bar{f}(\mc{X};\gamma)}\right\} 
\end{align*}
where $\widehat{\theta}_\ell( \mc{V} )$ is the estimator that would be constructed by the algorithm at stage $\ell$ with $\mc{X}_\ell = \mc{V}$.
For fixed $\mc{V} \subset \mc{X}$, $\ell \in \mathbb{N}$, $\tau_\ell$ actions are taken. Thus we apply Lemma~\ref{lem:rls_estimator_lemma} with $\tau = \tau_\ell$ and with regularization factor $\tau_\ell\gamma$, so that with probability at least $1-\frac{\delta}{2 \ell^2 |\mc{X}|}$ we have for any $x\in\mc{V}$
\begin{align*}
    |x^\top (\widehat{\theta}_\ell - \theta_*)|
    &\leq
    \|x\|_{\left(\sum_{i=1}^{\tau_\ell}\widetilde{x}_i\widetilde{x}_i^\top + \tau_\ell\gamma I\right)^{-1}}\left(\sqrt{\tau_\ell\gamma} \|\theta_*\|_2 + h \sqrt{\tau_\ell} + \sqrt{2\sigma^2 \log(4 \ell^2 |\mc{X}|/\delta)}\right)\\
    &\leq 2\|x\|_{(\tau_\ell A(\lambda_\ell) + \tau_\ell\gamma I)^{-1}}\left(\sqrt{\tau_\ell\gamma} \|\theta_*\|_2 + h \sqrt{\tau_\ell} + \sqrt{2\sigma^2 \log(4 \ell^2 |\mc{X}|/\delta)}\right)\\
    &= 2\|x\|_{(A(\lambda_\ell) + \gamma I)^{-1}}\left(\sqrt{\gamma} \|\theta_*\|_2 + h + \sqrt{ \frac{2 \sigma^2 \log(4 \ell^2 |\mc{X}|/\delta)}{\tau_\ell}}\right) \\
    &\leq 2\sqrt{f(\mc{V};\gamma)}\left(\sqrt{\gamma} \|\theta_*\|_2 + h + \epsilon_\ell / \sqrt{f(\mc{V};\gamma)}\right)\\
    &\leq 2\epsilon_\ell + 2\sqrt{\bar{f}(\mc{X};\gamma)}\left(\sqrt{\gamma} \|\theta_*\|_2 + h\right)
\end{align*}
Noting that $ \mc{E} := \bigcap_{\ell=1}^\infty \bigcap_{x \in \mc{X}_\ell} \mc{E}_{x,\ell}( \mc{X}_\ell ) $, the rest of the proof with the robust estimator applies here.
\end{proof}
\noindent The next lemma is similar to the one for the robust estimator, and the proof will follow the same argument as for the robust estimator.
\begin{lemma}
    For all $\ell \in \mathbb{N}$ we have $\max_{x \in \mc{X}_{\ell}}\mu_* - \mu_x \leq \max\{ 32 \epsilon_\ell, 32 (4\sqrt{\gamma} \|\theta_* \|_2 + 6h)\sqrt{ \bar{f}(\mc{X};\gamma)} \}$.
\end{lemma}
\begin{proof}
An arm $x \in \mc{X}_\ell$ is discarded (i.e., not in $\mc{X}_{\ell+1}$) if $\max_{x' \in \mc{X}_\ell} \langle x', \widehat{\theta} \rangle - \langle x, \widehat{\theta} \rangle > 8 \epsilon_\ell$. 
Let $\bar{\ell} := \max\{ \ell : \epsilon_\ell >  (4 \sqrt{\gamma} \|\theta_* \|_2 + 6h)\sqrt{ \bar{f}(\mc{X};\gamma)}\}$.
If $x_* = \arg\max_{x \in \mc{X}} \mu_x$ then $x_* \in \mc{X}_1$. Now if $x_* \in \mc{X}_\ell$ for some $\ell \leq \bar{\ell}$, then for any $x' \in \mc{X}_\ell$ we have
\begin{align*}
    \langle x', \widehat{\theta} \rangle - \langle x_*, \widehat{\theta} \rangle &\leq \langle x' -x_* ,\theta_* \rangle + 4\epsilon_\ell + 4 (\sqrt{\gamma} \|\theta_* \|_2 + h)\sqrt{ \bar{f}(\mc{X};\gamma)} \\
    &\leq \mu_x - \mu_{x_*} + 2h + 4\epsilon_\ell + 4 (\sqrt{\gamma} \|\theta_* \|_2 + h)\sqrt{ \bar{f}(\mc{X};\gamma)} \\
    &\leq 4\epsilon_\ell + (4 \sqrt{\gamma} \|\theta_* \|_2 + 6h)\sqrt{ \bar{f}(\mc{X};\gamma)} \\
    &\leq 8 \epsilon_\ell
\end{align*}
which implies $x_* \in \mc{X}_{\ell+1}$.
Moreover, suppose that $\ell \leq \bar{\ell}$ and there exists some $x \in \mc{X}_\ell$ such that $\mu_{*} - \mu_{x} > 16 \epsilon_\ell$, then
\begin{align*}
    \max_{x' \in \mc{X}_\ell} \langle x', \widehat{\theta} \rangle - \langle x, \widehat{\theta} \rangle &\geq \langle x_*, \widehat{\theta} \rangle - \langle x, \widehat{\theta} \rangle \\
    &\geq \langle x_*- x, {\theta}_* \rangle - 4\epsilon_\ell - 4 (\sqrt{\gamma} \|\theta_* \|_2 + h)\sqrt{ \bar{f}(\mc{X};\gamma)} \\
    &\geq \mu_* - \mu_x - 2h - 4\epsilon_\ell - 4 (\sqrt{\gamma} \|\theta_* \|_2 + h)\sqrt{ \bar{f}(\mc{X};\gamma)}\\
    &\geq \mu_* - \mu_x - 4\epsilon_\ell - (4 \sqrt{\gamma} \|\theta_* \|_2 + 6h)\sqrt{ \bar{f}(\mc{X};\gamma)} \\
    &\geq \mu_* - \mu_x - 8\epsilon_\ell  \\
    &> 8 \epsilon_\ell
\end{align*}
which implies $\max_{x \in \mc{X}_{\ell+1}} \mu_* - \mu_x \leq 16 \epsilon_\ell = 32 \epsilon_{\ell+1}$.
Because $\mc{X}_{\ell +1} \subseteq \mc{X}_\ell$ we have for $\ell > \bar{\ell}$ that
\begin{align*}
    \max_{x \in \mc{X}_{\ell}}\mu_* - \mu_x &\leq \max_{x \in \mc{X}_{\bar\ell+1}}\mu_* - \mu_x \\
    &\leq 32 \epsilon_{\bar\ell+1} \\
    &\leq 32 (4\sqrt{\gamma} \|\theta_* \|_2 + 6h)\sqrt{ \bar{f}(\mc{X};\gamma)}.
\end{align*}
Thus, $\max_{x \in \mc{X}_{\ell}}\mu_* - \mu_x \leq \max\{ 32 \epsilon_\ell, 32 (4\sqrt{\gamma} \|\theta_* \|_2 + 6h)\sqrt{ \bar{f}(\mc{X};\gamma)} \}$.
\end{proof}

We now compute the final regret bound. After $T$ steps of the algorithm, let $T_x$ denote the number of times arm $x$ is played.
Let $\Gamma = (4\sqrt{\gamma} \|\theta_* \|_2 + 
6h)\sqrt{ \bar{f}(\mc{X};\gamma)}$.
If $L$ is the final round reached after $T$ steps, we have
\begin{align*}
\sum_{x \in \mc{X}} (\mu_* - \mu_x) T_x 
&\leq \sum_{\ell=1}^L \max_{x \in \mc{X}_\ell} (\mu_* - \mu_x) \tau_\ell \\
&\leq \sum_{\ell=1}^L \tau_\ell \max\{ 32 \epsilon_\ell, 32 (4\sqrt{\gamma} \|\theta_* \|_2 + 6h)\sqrt{ \bar{f}(\mc{X};\gamma)} \} \\
&\leq \sum_{\ell=1}^L \tau_\ell \max\{ 32 \epsilon_\ell, 32 \Gamma \} \\
&\leq \sum_{\ell: \epsilon_\ell < \Gamma } 32 \Gamma \tau_\ell + 32\sum_{\ell: \epsilon_\ell \geq \Gamma }   \epsilon_\ell  \tau_\ell \\
&\leq \sum_{\ell: \epsilon_\ell < \Gamma } 32 \Gamma \tau_\ell + 32 \nu T + \sum_{\ell: \epsilon_\ell \geq \Gamma \vee  \nu }  32 \epsilon_\ell  \tau_\ell \\
&\leq \sum_{\ell: \epsilon_\ell < \Gamma } 32 \Gamma \tau_\ell + 32 \nu T + \sum_{\ell: \epsilon_\ell \geq \nu } 32 \epsilon_\ell  \tau_\ell \\
&\leq c\left(\Gamma  T + \nu T + \sum_{\ell: \epsilon_\ell \geq  \nu } \epsilon_\ell \cdot \left(2\sigma^2 \epsilon_\ell^{-2} f( \mc{X}_\ell ; \gamma) \log(4 \ell^2 |\mc{X}|/\delta) + \widetilde{d}(\gamma, \lambda_\ell)\right)\right) \\
&\leq c\left(\Gamma  T + \nu T + \left(2 \sigma^2 \bar{f}( \mc{X} ; \gamma) \log(4 \lceil \log_2(1/\nu ) \rceil^2 |\mc{X}|/\delta) + \widetilde{d}(\gamma)\right)   \sum_{\ell: \epsilon_\ell \geq \nu }  \epsilon_\ell^{-1} \right)\\
&\leq c\left(\Gamma  T + \nu T + \nu^{-1} \left(2 \sigma^2 \bar{f}( \mc{X} ; \gamma) \log(4 \lceil \log_2(1/\nu) \rceil^2 |\mc{X}|/\delta)\right) + 32\widetilde{d}(\gamma) \right).
\end{align*}
Where we denote $\widetilde{d}(\gamma) = \max_{\ell \leq L}\widetilde{d}(\gamma, \lambda_\ell)$. Choosing $\nu = \sqrt{\left(2 \sigma^2 \bar{f}( \mc{X} ; \gamma) \log(  |\mc{X}|/\delta)\right)/ T}$ and plugging $\Gamma$ back in yields
\begin{align*}
\sum_{x \in \mc{X}} (\mu_* - \mu_x) T_x &\leq c\left(\widetilde{d}(\gamma) + \sqrt{\bar{f}( \mc{X} ; \gamma)} \left( T (\sqrt{\gamma} \|\theta_* \|_2 + h) + \sqrt{\left(\sigma^2 \log(  |\mc{X}|\log(T)/\delta) \right) T} \right)\right).
\end{align*}
Choosing $\gamma = 1/T$ yields
\begin{align*}
\sum_{x \in \mc{X}} (\mu_* - \mu_x) T_x &\leq c\left(\widetilde{d}(\gamma) +  \sqrt{\bar{f}( \mc{X} ; 1/T)} \left( h T + \sqrt{\left((\|\theta_*\|^2 + \sigma^2) \log(  |\mc{X}|\log(T)/\delta) \right) T} \right)\right).
\end{align*}

\subsection{Sample complexity bound}
\begin{algorithm}[tb]
  \caption{PTR for Pure exploration}
  \label{alg:bai_rls}
\begin{algorithmic}
\State {\bfseries Input:} Finite sets $\mc{X} \subset \R^d$, $\mc{Z} \subset \R^d$, feature map $\phi$, confidence level $\delta \in (0, 1)$, regularization $\gamma$, sub-gaussian parameter $\sigma$, norm of model parameter $B$, bound on the misspecification noise $h$.\\
\State Let $\mc{Z}_1 \gets \mc{Z}, \ell\gets 1$ \\
\While{$|\mc{Z}_\ell|>1$}
    \State Let $\lambda_\ell\in \triangle_{\mc{X}}$ be a minimizer of $f(\lambda, \mc{Z}_\ell, \gamma)$ where
        \begin{align*}
            f( \mc{V}; \gamma) = \inf_{\lambda\in\triangle_\mc{X}}f(\lambda, \mc{V}, \gamma) = \inf_{\lambda\in\triangle_\mc{X}}
            \max_{v, v' \in \mc{V}} \| \phi(v) - \phi(v') \|_{(\sum_{x \in \mc{X}} \lambda_y \phi(x)\phi(x)^\top + \gamma I)^{-1}}^2
        \end{align*}
    \State Set $\epsilon_\ell = 2^{-\ell}$ and $\tau_\ell := \lceil \max\{2\sigma^2 \epsilon_\ell^{-2} f( \mc{Z}_\ell ; \gamma) \log(4 \ell^2 |\mc{Z}|/\delta),  \widetilde{d}(\gamma, \lambda_\ell)\}\rceil$
    \State Use the PTR procedure of section~\ref{sec:RKHS_rounding} to find sparse allocation $\{\widetilde{x}_i\}_{i=1}^{\tau_\ell} \subset \mc{X}_\ell$ from $\lambda_\ell$.
    \State Take each action $x \in \{\widetilde{x}_i\}_{i=1}^{\tau_\ell}$ with corresponding features $\Phi$ and rewards $Y$\\
    \State Compute $\widehat{\theta}_{\ell} = (\Phi^\top \Phi + \tau_\ell\gamma I)^{-1}\Phi^\top Y$
    \State 
    \begin{align*}
        \mc{Z}_{\ell+1} \gets \mc{Z}_\ell \setminus\big\{z \in \mc{Z}_\ell  : \max_{z' \in \mc{Z}_\ell} \langle \phi(z') - \phi(z), \widehat{\theta}_{\ell} \rangle > \epsilon_\ell \big\}
    \end{align*}
    \State $\ell \gets \ell + 1$ \\
\EndWhile
\State {\bfseries Output:} $\mc{Z}_{\ell}$
\end{algorithmic}
\end{algorithm}
For any $\mc{V} \subset \mc{Z}$ define $f(\mc{X},\mc{V};\gamma) = \min_{\lambda\in\triangle_\mc{X
}} \max_{z,z' \in \mc{V}} \| \phi(z) - \phi(z') \|_{( \sum_{x \in \mc{X}} \lambda_x \phi(x)\phi(x)^\top + \gamma I)^{-1}}^2$
\begin{theorem}\label{thm:samplecomplexity_rls}
With $\displaystyle z_* \in \arg\max_{z \in \mc{Z}} \langle z, \theta_* \rangle$,
fix any $\epsilon \geq \bar\epsilon$ where
\begin{align*}
    \bar{\epsilon} =  8 \min\{ \epsilon \geq 0 : 4(\sqrt{\gamma} \|\theta_* \|_2 + h)(2+ \sqrt{f(\mc{X},\{ z \in \mc{Z} : \langle \phi(z_*)- \phi(z), \theta_* \rangle \leq \epsilon \};\gamma)}) \leq \epsilon \}.
\end{align*}
Then with probability at least $1-\delta$, once the algorithm has taken at least $\tau$ samples where 
\begin{align*}
    \tau \leq c'\left(\widetilde{d}(\gamma)\log_2(8 (\Delta \vee \bar\epsilon )^{-1}) + \lceil \log_2(8 (\Delta \vee \bar\epsilon )^{-1} )\rceil \sigma^2 \log(2\lceil \log_2(8 (\Delta \vee \bar\epsilon )^{-1} )\rceil^2|\mc{Z}|^2/\delta) \rho^*(\gamma,\bar\epsilon)\right)
\end{align*}
we have that $\mu_{\widehat{z}} \geq \max_{z' \in \mc{Z}} - \epsilon$ where $\widehat{z}$ is any arm in the set $\mc{Z}_\ell$ under consideration after $\tau$ pulls and 
\begin{align}
    \rho^*(\gamma,\epsilon) \!&=\! \inf_{\lambda \in \Delta_{\mathcal{X}}}\sup_{z\in \mathcal{Z}}\frac{\|\phi(z_*) \!-\! \phi(z)\|^2_{A^{(\gamma)}(\lambda)^{-1}}}{\max\{\epsilon^2,\langle \theta_*, \phi(z_*)\!-\!\phi(z) \rangle^2)\}}
\end{align}
and $\widetilde{d}(\gamma) = \max_{\ell \leq \lceil\log_2(8 (\Delta \vee \epsilon )^{-1})\rceil}\widetilde{d}(\gamma, \lambda_\ell) \leq \max_{\mc{V}\subset \mc{X}}\sup_{\lambda\in\triangle_\mc{V}}\widetilde{d}(\gamma, \lambda)$.
\end{theorem}

We first prove the following intermediate result.
\begin{theorem}
Recall that we defined
\begin{align*}
    \bar{\epsilon} =  8 \min\{ \epsilon \geq 0 : 4(\sqrt{\gamma} \|\theta_* \|_2 + h)(2+ \sqrt{f(\mc{X},\{ z \in \mc{Z} : \langle \phi(z_*)- \phi(z), \theta_* \rangle \leq \epsilon \};\gamma)}) \leq \epsilon \}.
\end{align*}
Then $\max_{z \in \mc{Z}_{\ell}}\mu_* - \mu_z \leq 16 \max\{ \epsilon_\ell, \bar\epsilon \}$ for all $\ell \geq 0$ with probability at least $1-\delta$.
\end{theorem}

Define the event
\begin{align*}
    \mc{E} := \bigcap_{\ell=1}^\infty \bigcap_{z,z' \in \mc{Z}_\ell} \left\{  |\langle \phi(z) - \phi(z'), \widehat{\theta}_\ell - \theta_* \rangle| \leq 2\epsilon_\ell + 2\left(\sqrt{\gamma} \|\theta_*\|_2 + h\right)\sqrt{f(\mc{X},\mc{Z}_\ell;\gamma)} \right\}
\end{align*}
\begin{lemma}
We have $\P( \mc{E} ) \geq 1-\delta$.
\end{lemma}
\begin{proof}
For any $\mc{V} \subseteq \mc{Z}$ and $x \in \mc{}$ define
\begin{align*}
    \mc{E}_{z,z',\ell}( \mc{V} ) = \left\{ |\langle \phi(z)-\phi(z'), \widehat{\theta}_\ell(\mc{V}) - \theta_* \rangle| \leq 2\epsilon_\ell + 2\left(\sqrt{\gamma} \|\theta_*\|_2 + h\right)\sqrt{f(\mc{X},\mc{V};\gamma)} \right\} 
\end{align*}
where $\widehat{\theta}_\ell( \mc{V} )$ is the estimator that would be constructed by the algorithm at stage $\ell$ with $\mc{Z}_\ell = \mc{V}$.
For fixed $\mc{V} \subset \mc{Z}$, $\ell \in \mathbb{N}$ and $x \in \mc{V}$, $\tau_\ell$ actions are taken. Thus we apply Lemma~\ref{lem:rls_estimator_lemma} with $\tau = \tau_\ell$ and with regularization factor $\tau_\ell\gamma$, so that with probability at least $1-\frac{\delta}{2 \ell^2 |\mc{Z}|}$ we have for any $z, z' \in \mc{V}$
\begin{align*}
    &|(\phi(z)-\phi(z'))^\top (\widehat{\theta}_\ell - \theta_*)|\\
    &\leq
    2\|\phi(z)-\phi(z')\|_{\left(\sum_{i=1}^{\tau_\ell}\phi(\widetilde{x}_i)\phi(\widetilde{x}_i)^\top + \tau_\ell\gamma I\right)^{-1}} \left(\sqrt{\tau_\ell\gamma} \|\theta_*\|_2 + h \sqrt{\tau_\ell} + \sqrt{2\sigma^2 \log(4 \ell^2 |\mc{Z}|/\delta)}\right)\\
    &\leq 2\|\phi(z)-\phi(z')\|_{(\tau_\ell A(\lambda) + \tau_\ell\gamma I)^{-1}}\left(\sqrt{\tau_\ell\gamma} \|\theta_*\|_2 + h \sqrt{\tau_\ell} + \sqrt{2\sigma^2 \log(4 \ell^2 |\mc{Z}|/\delta)}\right)\\
    &= 2\|\phi(z)-\phi(z')\|_{(A(\lambda) + \gamma I)^{-1}}\left(\sqrt{\gamma} \|\theta_*\|_2 + h + \sqrt{ \frac{2 \sigma^2 \log(4 \ell^2 |\mc{Z}|/\delta)}{\tau_\ell}}\right) \\
    &\leq 2\sqrt{f(\mc{X},\mc{V};\gamma)}\left(\sqrt{\gamma} \|\theta_*\|_2 + h + \epsilon_\ell / \sqrt{f(\mc{X},\mc{V};\gamma)}\right)\\
    &= 2\epsilon_\ell + 2\sqrt{f(\mc{X},\mc{V};\gamma)}\left(\sqrt{\gamma} \|\theta_*\|_2 + h\right)
\end{align*}
Noting that $ \mc{E} := \bigcap_{\ell=1}^\infty \bigcap_{z,z' \in \mc{Z}_\ell} \mc{E}_{z,z',\ell}( \mc{Z}_\ell ) $, the rest of the proof with the robust estimator applies here.
\end{proof}

\begin{lemma}
    For all $\ell \in \mathbb{N}$ we have $\max_{z \in \mc{Z}_{\ell}}\mu_* - \mu_z \leq 16 \max\{ \epsilon_\ell, \epsilon_{\bar\ell} \}$.
\end{lemma}
\begin{proof}
An arm $z \in \mc{Z}_\ell$ is discarded (i.e., not in $\mc{Z}_{\ell+1}$) if $\max_{z' \in \mc{Z}_\ell} \langle \phi(z')-\phi(z), \widehat{\theta} \rangle > 4 \epsilon_\ell$.

Define $S_1 = \mc{Z}$ and $S_{\ell+1} = \{z \in S_\ell : \langle \phi(z_*)- \phi(z), \theta_* \rangle \leq 6\epsilon_\ell + 2(\sqrt{\gamma} \|\theta_* \|_2 + h) \sqrt{f(\mc{X},S_\ell;\gamma)} \} $.
Define 
\begin{align*}
    \bar\ell &= \max \{ \ell : (\sqrt{\gamma} \|\theta_* \|_2 + h)(2+ \sqrt{f(\mc{X},S_\ell;\gamma)}) \leq \epsilon_\ell \} .
\end{align*}

We will show $\{z_* \in \mc{Z}_\ell \} \cap \{ \mc{Z}_\ell \subset S_\ell \} \cap \{ \ell \leq \bar\ell\} \implies \{z_* \in \mc{Z}_{\ell+1} \} \cap \{ \mc{Z}_{\ell+1} \subset S_{\ell+1} \}$.
Noting that $\{z_* \in \mc{Z}_\ell \} \cap \{ \mc{Z}_\ell \subset S_\ell \}$ holds for $\ell=1$, we will assume an inductive hypothesis  of this condition for some $\ell \leq \bar\ell$.

First we will show $\{z_* \in \mc{Z}_\ell \} \cap \{ \mc{Z}_\ell \subset S_\ell \} \cap \{ \ell \leq \bar\ell\} \implies \{z_* \in \mc{Z}_{\ell+1} \}$.
On $\{z_* \in \mc{Z}_\ell \} \cap \{ \mc{Z}_\ell \subset S_\ell \} \cap \{ \ell \leq \bar\ell\}$, we have for any $z' \in \mc{Z}_\ell$ that
\begin{align*}
    \langle \phi(z')-\phi(z_*), \widehat{\theta} \rangle &\leq \langle \phi(z')-\phi(z_*), \theta_* \rangle + 2\epsilon_\ell + 2(\sqrt{\gamma} \|\theta_* \|_2 + h) \sqrt{f(\mc{X},\mc{Z}_\ell;\gamma)} \\
    &\leq \mu_{z'} - \mu_{z_*} + 2h + 2\epsilon_\ell + 2(\sqrt{\gamma} \|\theta_* \|_2 + h) \sqrt{f(\mc{X},\mc{Z}_\ell;\gamma)} \\
    &\leq 2\epsilon_\ell + 2(\sqrt{\gamma} \|\theta_* \|_2 + h)(2+ 
    \sqrt{f(\mc{X},\mc{Z}_\ell;\gamma)})  \\
    &\leq 2\epsilon_\ell + 2(\sqrt{\gamma} \|\theta_* \|_2 + h)(2 + \sqrt{f(\mc{X},S_\ell;\gamma)}) \\
    &\leq 4 \epsilon_\ell
\end{align*}
which implies $z_*$ is not eliminated, that is, $z_* \in \mc{Z}_{\ell+1}$.
The second-to-last inequality follows from
\begin{align*}
    f(\mc{X},\mc{Z}_\ell;\gamma) &= \inf_\lambda \max_{z,z' \in \mc{Z}_\ell} \| \phi(z) - \phi(z') \|_{(\sum_{x \in \mc{X}} \lambda_x \phi(x)\phi(x)^\top + \gamma I)^{-1}}^2 \\
    &\leq \inf_\lambda \max_{z,z' \in S_\ell} \| \phi(z) - \phi(z') \|_{(\sum_{x \in \mc{X}} \lambda_x \phi(x)\phi(x)^\top + \gamma I)^{-1}}^2 \\
    &= f(\mc{X},S_\ell;\gamma).
\end{align*}

Now we will show $\{z_* \in \mc{Z}_\ell \} \cap \{ \mc{Z}_\ell \subset S_\ell \} \cap \{ \ell \leq \bar\ell\} \implies \{ \mc{Z}_{\ell+1} \subset S_{\ell+1} \}$.
For any $z \in \mc{Z}_\ell \cap S_{\ell+1}^c$ we have
\begin{align*}
    \max_{z' \in \mc{Z}_\ell}\langle \phi(z')-\phi(z), \widehat{\theta} \rangle &\geq \langle \phi(z_*)-\phi(z), \widehat{\theta} \rangle \\
    &\geq \langle \phi(z_*)-\phi(z), \theta_* \rangle - 2\epsilon_\ell - 2(\sqrt{\gamma} \|\theta_* \|_2 + h) \sqrt{f(\mc{X},\mc{Z}_\ell;\gamma)}\\
    &> 6\epsilon_\ell + 2(\sqrt{\gamma} \|\theta_* \|_2 + h) \sqrt{f(\mc{X},S_\ell;\gamma)}  - 2\epsilon_\ell -  2(\sqrt{\gamma} \|\theta_* \|_2 + h) \sqrt{f(\mc{X},\mc{Z}_\ell;\gamma)} \\
    &\geq 4\epsilon_\ell
\end{align*}
which implies $z \not\in \mc{Z}_{\ell+1}$, and $\mc{Z}_{\ell+1} \subset S_{\ell+1}$.

Thus, for $\ell \leq \bar\ell$ we have
\begin{align*}
    \max_{z \in \mc{Z}_{\ell}}\mu_* - \mu_z &\leq \max_{z \in \mc{Z}_{\ell}}  \langle \phi(z_*)-\phi(z), \theta_* \rangle + 2h \\
    &\leq \max_{z \in \mc{S}_{\ell}}  \langle \phi(z_*)-\phi(z), \theta_* \rangle + 2h \\
    &\leq 6\epsilon_{\ell-1} + 2h + 2 (\sqrt{\gamma} \|\theta_* \|_2 + h) \sqrt{f(\mc{X},S_{\ell-1};\gamma)} \\
    &\leq 6\epsilon_{\ell-1} + 2(\sqrt{\gamma} \|\theta_* \|_2 + h)(2+ \sqrt{f(\mc{X},S_{\ell-1};\gamma)}) \\
    &\leq 8 \epsilon_{\ell-1} = 16 \epsilon_\ell.
\end{align*}
And because $\mc{Z}_{\ell+1} \subseteq \mc{Z}_\ell$ we always have that $\max_{z \in \mc{Z}_{\ell}}\mu_* - \mu_z \leq 16 \max\{ \epsilon_\ell, \epsilon_{\bar\ell} \}$. Note that 
\begin{align*}
    \bar\ell &= \max \{ \ell : (\sqrt{\gamma} \|\theta_* \|_2 + h)(2 + \sqrt{f(\mc{X},S_\ell;\gamma)}) \leq \epsilon_\ell \} \\
    &\geq \max \{ \ell : (\sqrt{\gamma} \|\theta_* \|_2 + h)(2+ \sqrt{f(\mc{X},\{ z \in \mc{Z} : \langle \phi(z_*)- \phi(z), \theta_* \rangle \leq 4\epsilon_\ell \};\gamma)}) \leq \epsilon_\ell \} \\
    &= \max \{ \ell : 4(\sqrt{\gamma} \|\theta_* \|_2 + h)(2 + \sqrt{f(\mc{X},\{ z \in \mc{Z} : \langle \phi(z_*)- \phi(z), \theta_* \rangle \leq 4\epsilon_\ell \};\gamma)}) \leq 4 \epsilon_\ell \} \\
    &= -2 + \max \{ \ell : 4(\sqrt{\gamma} \|\theta_* \|_2 + h)(2+ \sqrt{f(\mc{X},\{ z \in \mc{Z} : \langle \phi(z_*)- \phi(z), \theta_* \rangle \leq \epsilon_\ell \};\gamma)}) \leq  \epsilon_\ell \} \\
    &\geq -3 - \log_2( \min\{ \epsilon > 0 :  4(\sqrt{\gamma} \|\theta_* \|_2 + h)(2+ \sqrt{f(\mc{X},\{ z \in \mc{Z} : \langle \phi(z_*)- \phi(z), \theta_* \rangle \leq \epsilon \};\gamma)}) \leq \epsilon \})
\end{align*}
which defines $\bar{\epsilon}$.
\end{proof}

Denoting $\widetilde{d}(\gamma) = \max_{\ell \leq \lceil \log_2(8 (\Delta \vee \bar\epsilon )^{-1} )\rceil}\widetilde{d}(\gamma, \lambda_\ell)$, the sample complexity to return an $\bar\epsilon$-good arm is equal to
\begin{align*}
    &\sum_{\ell=1}^{\lceil \log_2(8 (\Delta \vee \bar\epsilon )^{-1} )\rceil}\tau_\ell\\
    &\leq \sum_{\ell=1}^{\lceil \log_2(8 (\Delta \vee \bar\epsilon )^{-1} )\rceil} (2 \epsilon_\ell^{-2}(\gamma) f(\mc{X},\mc{Z}_\ell;\gamma) \sigma^2 \log(2\ell^2|\mc{Z}|^2/\delta)+\widetilde{d}(\gamma, \lambda_\ell)) \\
    &\leq c\left(\widetilde{d}(\gamma)\log_2(8 (\Delta \vee \bar\epsilon )^{-1}) + \sigma^2 \log(2\lceil \log_2(8 (\Delta \vee \bar\epsilon )^{-1} )\rceil^2|\mc{Z}|^2/\delta) \sum_{\ell=1}^{\lceil \log_2(8 (\Delta \vee \bar\epsilon )^{-1} )\rceil} \epsilon_\ell^{-2} f(\mc{X},\mc{Z}_\ell;\gamma)\right)  \\
    &\leq c\left(\widetilde{d}(\gamma)\log_2(8 (\Delta \vee \bar\epsilon )^{-1}) + \sigma^2 \log(2\lceil \log_2(8 (\Delta \vee \bar\epsilon )^{-1} )\rceil^2|\mc{Z}|^2/\delta) \sum_{\ell=1}^{\lceil \log_2(8 (\Delta \vee \bar\epsilon )^{-1} )\rceil} \epsilon_\ell^{-2} f(\mc{X},S_\ell;\gamma)\right) \\
    &\leq c'\left(\widetilde{d}(\gamma)\log_2(8 (\Delta \vee \bar\epsilon )^{-1}) + \lceil \log_2(8 (\Delta \vee \bar\epsilon )^{-1} )\rceil \sigma^2 \log(2\lceil \log_2(8 (\Delta \vee \bar\epsilon )^{-1} )\rceil^2|\mc{Z}|^2/\delta) \rho^*(\gamma,\bar\epsilon)\right)
\end{align*}
where the last line follows from
\begin{align*}
    \rho^*(\gamma,\bar\epsilon) &= \inf_{\lambda \in \Delta_{\mathcal{X}}}\sup_{z\in \mathcal{Z}}\frac{\|\phi(z_*) - \phi(z)\|^2_{\left(\sum_{x\in \mathcal{X}}\lambda_x \phi(x)\phi(x)^\top + \gamma I\right)^{-1}}}{\max\{\bar\epsilon^2,\langle\phi(z_*) - \phi(z), \theta_*\rangle^2)\}} \\
    &= \inf_{\lambda \in \Delta_{\mathcal{X}}}\sup_{\ell\leq \lceil \log_2(8 (\Delta \vee \bar\epsilon )^{-1} )\rceil}\sup_{z\in S_\ell}\frac{\|\phi(z_*) - \phi(z)\|^2_{\left(\sum_{x\in \mathcal{X}}\lambda_x \phi(x)\phi(x)^\top + \gamma I\right)^{-1}}}{\max\{\bar\epsilon^2,((\phi(z_*) - \phi(z))^\top \theta_*)^2\}} \\
    &\geq \inf_{\lambda \in \Delta_{\mathcal{X}}}\frac{1}{\lceil \log_2(8 (\Delta \vee \bar\epsilon )^{-1} )\rceil}\sum_{\ell = 1}^{\lceil \log_2(8 (\Delta \vee \bar\epsilon )^{-1} )\rceil}\sup_{z\in S_\ell}\frac{\|\phi(z_*) - \phi(z)\|^2_{\left(\sum_{x\in \mathcal{X}}\lambda_x \phi(x)\phi(x)^\top + \gamma I\right)^{-1}}}{\max\{\bar\epsilon^2,((\phi(z_*) - \phi(z))^\top \theta_*)^2\}} \\
    &\geq \frac{1}{\lceil \log_2(8 (\Delta \vee \bar\epsilon )^{-1} )\rceil}\sum_{\ell = 1}^{\lceil \log_2(8 (\Delta \vee \bar\epsilon )^{-1} )\rceil}\inf_{\lambda \in \Delta_{\mathcal{X}}}\sup_{z\in S_\ell}\frac{\|\phi(z_*) - \phi(z)\|^2_{\left(\sum_{x\in \mathcal{X}}\lambda_x \phi(x)\phi(x)^\top + \gamma I\right)^{-1}}}{\max\{\bar\epsilon^2,((\phi(z_*) - \phi(z))^\top \theta_*)^2\}} \\
    &\geq \frac{1}{4\lceil \log_2(8 (\Delta \vee \bar\epsilon )^{-1} )\rceil}\sum_{\ell = 1}^{\lceil \log_2(8 (\Delta \vee \bar\epsilon )^{-1} )\rceil}2^{2\ell}\inf_{\lambda \in \Delta_{\mathcal{X}}}\sup_{z\in S_\ell}\|\phi(z_*) - \phi(z)\|^2_{\left(\sum_{x\in \mathcal{X}}\lambda_x \phi(x)\phi(x)^\top + \gamma I\right)^{-1}} \\
    &\geq \frac{1}{16\lceil \log_2(8 (\Delta \vee \bar\epsilon )^{-1} )\rceil}\sum_{\ell = 1}^{\lceil \log_2(8 (\Delta \vee \bar\epsilon )^{-1} )\rceil}2^{2\ell}\inf_{\lambda \in \Delta_{\mathcal{X}}}\sup_{z, z'\in S_\ell}\|\phi(z) - \phi(z')\|^2_{\left(\sum_{x\in \mathcal{X}}\lambda_x \phi(x)\phi(x)^\top + \gamma I\right)^{-1}} \\
    &= \frac{1}{16\lceil \log_2(8 (\Delta \vee \bar\epsilon )^{-1} )\rceil}\sum_{\ell = 1}^{\lceil \log_2(8 (\Delta \vee \bar\epsilon )^{-1} )\rceil}2^{2\ell}f(\mc{X},S_\ell;\gamma)  \;.
\end{align*}

\section{Related work results}
\begin{lemma}\label{g-opt-effective-dim}
If $\lambda^* \in \arg\max_{\lambda\in\triangle_\mc{V}} f(\lambda)$ where $f(\lambda) = \log\left(\det\left(\sum_{x\in\mc{X}}\lambda_x \phi(x)\phi(x)^\top + \gamma I\right)\right)$,
then 
\begin{align*}
    \sup_{x\in\mc{X}}\|\phi(x)\|^2_{A^{(\gamma)}(\lambda^*)^{-1}}  &= \sum_{x\in\mc{X}}\lambda^*_x\|\phi(x)\|^2_{A^{\gamma}(\lambda^*)^{-1}}\\
    &= \text{Trace}\left(A(\lambda^*)(A(\lambda^*) + \gamma I)^{-1}\right)\\
    &= \text{Trace}\left(K_{\lambda^*} (K_{\lambda^*} + \gamma I)^{-1}\right)
\end{align*}
\end{lemma}
\begin{proof}
We first state that
\begin{align*}
    \sum_{x\in\mc{X}}\lambda^*_x\|\phi(x)\|^2_{A^{\gamma}(\lambda^*)^{-1}} &= \sum_{x\in\mc{X}}\lambda^*_x \phi(x)^\top A^{\gamma}(\lambda^*)^{-1} \phi(x) \\
    &= \text{Trace}\left(\sum_{x\in\mc{X}}\lambda^*_x \phi(x)^\top A^{\gamma}(\lambda^*)^{-1} \phi(x)\right)\\ 
    &= \text{Trace}\left(\sum_{x\in\mc{X}}\lambda^*_x \phi(x)\phi(x)^\top  A^{\gamma}(\lambda^*)^{-1}\right)
    \\ 
    &= \text{Trace}\left(A(\lambda^*)(A(\lambda^*) + \gamma I)^{-1}\right)
\end{align*}
This implies
\begin{align*}
    \sup_{x\in\mc{X}}\|\phi(x)\|^2_{A^{(\gamma)}(\lambda^*)^{-1}} \geq \sum_{x\in\mc{X}}\lambda^*_x\|\phi(x)\|^2_{A^{\gamma}(\lambda^*)^{-1}} = \text{Trace}\left(A(\lambda^*)(A(\lambda^*) + \gamma I)^{-1}\right)
\end{align*}
Further, one can compute that 
\begin{align*}
    [\nabla_{\lambda}f(\lambda^*)]_x = \text{Trace}\left(A^{(\gamma)}(\lambda^*)^{-1}xx^\top \right) = \|\phi(x)\|^2_{A^{(\gamma)}(\lambda^*)^{-1}}
\end{align*}
And last, $\lambda^*$ satisfies the first order conditions on $f$: for any $\lambda \in \triangle_\mc{X}$
\begin{align*}
    0 &\geq \langle \nabla_{\lambda}f(\lambda^*), \lambda - \lambda^* \rangle \\
    &= \sum_{x\in\mc{X}}\lambda_x\|\phi(x)\|^2_{A^{\gamma}(\lambda^*)^{-1}} - \sum_{x\in\mc{X}}\lambda^*_x\|\phi(x)\|^2_{A^{\gamma}(\lambda^*)^{-1}} \\
    &= \sum_{x\in\mc{X}}\lambda_x\|\phi(x)\|^2_{A^{\gamma}(\lambda^*)^{-1}} - \text{Trace}\left(A(\lambda^*)(A(\lambda^*) + \gamma I)^{-1}\right)
\end{align*}
Choosing $\lambda$ to be a Dirac at $\arg\max_{x\in\mc{X}}\|\phi(x)\|^2_{A^{\gamma}(\lambda^*)^{-1}}$, we get to 
\begin{align*}
    \max_{x\in\mc{X}}\|\phi(x)\|^2_{A^{\gamma}(\lambda^*)^{-1}} \leq \text{Trace}\left(A(\lambda^*)(A(\lambda^*) + \gamma I)^{-1}\right).
\end{align*}

Hence the result of the lemma.
\end{proof}
\begin{lemma}
We can lower bound $\gamma_T$ the notion of information gain from \cite{srinivas2009gaussian} as
\begin{align*}
    \gamma_T \geq \frac{2}{3}\max_{\mc{V}\subset\mc{X}}\inf_{\lambda\in \triangle_\mc{V}} \sup_{x\in\mc{V}}\|\phi(x)\|^2_{\left( \sum_{x\in \mc{V}} \lambda_x \phi(x)\phi(x)^\top + \gamma/T I\right)^{-1}} + d\log(\gamma)\;.
\end{align*}
\end{lemma}
\begin{proof}
Recall the definition of \cite{srinivas2009gaussian} notion of information gain:
\begin{align*}
    \gamma_T &:= \sup_{\lambda\in\triangle_\mc{X}}\log (\det\left(T K_\lambda + \gamma I\right)) 
\end{align*}
where $K_\lambda$ is defined in Section~\ref{sec:computing_with_kernel_evaluations}. Note that the case where we have an infinite dimensional RKHS and $\phi$ is any feature map reduces to the finite one with $\phi $ being the identity map by computing $\Phi_\lambda \in \R^{|\mc{X}|\times|\mc{X}|}$ such that $K_\lambda = \Phi_\lambda \Phi_\lambda^\top$ and then looking at the (finite dimension) columns of $\Phi_\lambda$. So we can write without loss of generality 
\begin{align*}
    \gamma_T = \sup_{\lambda\in \triangle_\mc{X}}\log\left(\det\left(T\sum_{x\in\mc{X}} \lambda_x x x^\top + \gamma I\right)\right)
\end{align*}
Thus
\begin{align*}
    \gamma_T = \sup_{\lambda \in \triangle_\mc{X}}\log\left(\det\left(T\sum_{x\in\mc{X}}\lambda_x xx^\top + \gamma I\right)\right) \geq \sup_{\mc{V}\subset\mc{X}}\sup_{\lambda \in \triangle_\mc{V}}\log\left(\det\left(T\sum_{x\in\mc{V}}\lambda_x xx^\top + \gamma I\right)\right)
\end{align*}
Fix for now $\mc{V}\subset\mc{X}$ and let $\lambda^*\in \triangle_\mc{V}$ be such that 
\begin{align*}
    \lambda^* \in \arg\max_{\lambda \in \triangle_\mc{V}}\log\left(\det\left(T\sum_{x\in\mc{V}}\lambda_x xx^\top + \gamma I\right)\right) = \arg\max_{\lambda \in\triangle_\mc{V}}\log\left(\det\left(\sum_{x\in\mc{V}}\lambda_x xx^\top + \gamma/T I\right)\right)
\end{align*}
Inspired from equation 19.9 of \cite{lattimore2020bandit}, we write for some $x_0\in\mc{V}$ 
\begin{align*}
    &\det\left(T\sum_{x\in\mc{V}}\lambda^*_x xx^\top + \gamma I\right)\\ 
    &= \det\left(T\sum_{x\in\mc{V}\setminus \{x_0\}}\lambda^*_x xx^\top + \gamma I + T\lambda^*_{x_0} x_0 x_0^\top\right)\\
    &= \det\left(T\sum_{x\in\mc{V}\setminus \{x_0\}}\lambda^*_x xx^\top + \gamma I\right)\det\left(I + \left(T\sum_{x\in\mc{V}\setminus \{x_0\}}\lambda^*_x xx^\top + \gamma I\right)^{-1/2}T\lambda^*_{x_0} x_0 x_0^\top\left(T\sum_{x\in\mc{V}\setminus \{x_0\}}\lambda^*_x xx^\top + \gamma I\right)^{-1/2}\right)\\
    &= \det\left(T\sum_{x\in\mc{V}\setminus \{x_0\}}\lambda^*_x xx^\top + \gamma I\right)\left(1+T\lambda^*_{x_0}\|x_0\|^2_{\left(T\sum_{x\in\mc{V}\setminus \{x_0\}}\lambda^*_x xx^\top + \gamma I\right)^{-1}}\right)\\
    &\geq \det\left(T\sum_{x\in\mc{V}\setminus \{x_0\}}\lambda^*_x xx^\top + \gamma I\right)\left(1+T\lambda^*_{x_0}\|x_0\|^2_{\left(T\sum_{x\in\mc{V}}\lambda^*_x xx^\top + \gamma I\right)^{-1}}\right)
\end{align*}
We can now iterate with all the remaining $x\in\mc{V}\setminus \{x_0\}$, to get 
\begin{align*}
    \det\left(T\sum_{x\in\mc{V}}\lambda^*_x xx^\top + \gamma I\right) \geq \det\left(\gamma I\right)\prod_{x_0\in\mc{V}} \left(1+T\lambda^*_{x_0}\|x_0\|^2_{\left(T\sum_{x\in\mc{V}}\lambda^*_x xx^\top + \gamma I\right)^{-1}}\right)
\end{align*}
equivalent to
\begin{align*}
    \log\det\left(T\sum_{x\in\mc{V}}\lambda^*_x xx^\top + \gamma I\right) \geq \log\det\left(\gamma I\right)+\sum_{x\in\mc{V}} \log\left(1+T \lambda^*_{x}\|x\|^2_{\left(T\sum_{x\in\mc{V}}\lambda^*_x xx^\top + \gamma I\right)^{-1}}\right)
\end{align*}
We know that if $x\geq0$ 
holds $\log(1+x) \geq 2x/(2+x)$. Note that $\|x\|^2_{\left(\sum_{x'\in\mc{V}} \lambda_{x'} x'x'^\top + \gamma/T I\right)^{-1}} \leq 1$ always holds:
\begin{align*}
    &\|x\|^2_{\left(\sum_{x'\in\mc{V}} \lambda_{x'} x'x'^\top + \gamma/T I\right)^{-1}} = \|x\|^2_{\left(\sum_{x'\in\mc{V}\setminus\{x\}} \lambda_{x'} x'x'^\top + \gamma/T I + xx^\top\right)^{-1}} = \alpha - \alpha^2 / (1+\alpha) = \alpha / (1+\alpha) \leq 1\;.
\end{align*}
where we used Sherman–Morrison formula and defined $\alpha = \|x\|^2_{\left(\sum_{x'\in\mc{V}\setminus\{x\}} \lambda_{x'} x'x'^\top + \gamma/T I\right)^{-1}}$. Thus holds $0\leq T\lambda^*_x\|x\|^2_{\left(T\sum_{x'\in\mc{V}} \lambda^*_{x'} x'x'^\top + \gamma I\right)^{-1}} \leq 1$. So 
\begin{align*}
    \sum_{x\in\mc{V}}\log\left(1+T \lambda^*_x\|x\|^2_{\left(T\sum_{x\in\mc{V}}\lambda^*_x xx^\top + \gamma I\right)^{-1}}\right) \geq \frac{2}{2+1} \sum_{x\in\mc{V}}  T\lambda^*_x\|x\|^2_{\left(T\sum_{x\in\mc{V}}\lambda^*_x xx^\top + \gamma I\right)^{-1}}
\end{align*}
And thus
\begin{align*}
    \frac{3}{2}\log\left(\frac{\det\left(T\sum_{x\in\mc{V}}\lambda^*_x xx^\top + \gamma I\right)}{\det(\gamma I)}\right)
    &\geq \sum_{x\in\mc{V}}  T\lambda^*_x\|x\|^2_{\left(T\sum_{x\in\mc{V}}\lambda^*_x xx^\top + \gamma I\right)^{-1}}\\
    &= \sum_{x\in\mc{V}}\lambda^*_x \|x\|^2_{(\sum_{x\in\mc{V}} \lambda^*_x xx^\top + \gamma/T I)^{-1}} \\
    &= \sup_{x\in\mc{V}} \|x\|^2_{(\sum_{x\in\mc{V}} \lambda^*_x xx^\top + \gamma/T I)^{-1}}  \;.
\end{align*}
Where the last equality comes from lemma~\ref{g-opt-effective-dim} with  $\lambda^*\in \arg\max_{\lambda \in\triangle_\mc{V}}\log\left(\det\left(\sum_{x\in\mc{V}}\lambda_x xx^\top + \gamma/T I\right)\right)$.\\
So to summarize
\begin{align*}
    \gamma_T &:= \sup_{\lambda \in \triangle_\mc{X}}\log\left(\det\left(T\sum_{x\in\mc{X}}\lambda_x \phi(x)\phi(x)^\top + \gamma I\right)\right) \\
    &\geq \sup_{\mc{V}\subset\mc{X}}\sup_{\lambda \in \triangle_\mc{V}}\log\left(\det\left(T\sum_{x\in\mc{V}}\lambda_x \phi(x)\phi(x)^\top + \gamma I\right)\right)\\
    &= \sup_{\mc{V}\subset\mc{X}}\log\left(\frac{\det\left(T\sum_{x\in\mc{V}}\lambda^*_x \phi(x)\phi(x)^\top + \gamma I\right)}{\det(\gamma I)}\right) + \log(\det(\gamma I))\\
    &\geq \sup_{\mc{V}\subset\mc{X}}\frac{2}{3}\sup_{x\in\mc{V}} \|x\|^2_{(A(\lambda^*) + \gamma/T I)^{-1}} + \log(\det(\gamma I))\\
    &\geq \frac{2}{3}\sup_{\mc{V}\subset\mc{X}}\inf_{\lambda\in\triangle_\mc{V}}\sup_{x\in\mc{X}} \|x\|^2_{(A(\lambda) + \gamma/T I)^{-1}} + d\log(\gamma)
\end{align*}
\end{proof}

\begin{corollary}[Consequence of Theorem 1 of \cite{degenne2020gamification}]
Let $\tau_\delta$ be the expected number of sample needed to find the best arm with probability at least $1-\delta$. For any $\theta_* \in \mathcal{E}$ we have 
\begin{align*}
    \liminf_{\delta \rightarrow 0} \frac{\mathbb{E}_{\theta_*}[\tau_\delta]}{\log(1/\delta)} \geq T^*(\theta_*)
\end{align*}
where
\begin{align*}
T^{*-1}(\theta_*) = \max_{\lambda\in \Delta_{\mathcal{X}}}\inf_{x' \neq x_*}\sup_{\gamma \geq 0} F(\lambda, x', \gamma, \theta_*)
\end{align*}
\begin{align*}
    F(\lambda, x', \gamma, \theta_*) &= \frac{\max\{(x'-x_*)^\top (A(\lambda) + \gamma I)^{-1}A(\lambda) \theta_*, 0\}^2}{2\|x'-x_*\|^2_{(A(\lambda) + \gamma I)^{-1}}}+ \frac{\gamma}{2} \left(\|\theta_*\|_{(A(\lambda) + \gamma I)^{-1}A(\lambda)}^2 - R^2\right)
\end{align*}
\end{corollary}
\begin{proof}
Recall theorem 1 of \cite{degenne2020gamification}. For any $\theta_* \in \mathcal{E}$ we have 
\begin{align*}
    \liminf_{\delta \rightarrow 0} \frac{\mathbb{E}_{\theta_*}[\tau_\delta]}{\log(1/\delta)} \geq T^*(\theta_*)
\end{align*}
where we define the characteristic time through
\begin{align*}
    T^{*-1}(\theta_*) := \max_{\lambda\in \Delta_{\mathcal{X}}}\inf_{ \theta \in  \overline{S}_{x_*}} \|\theta - \theta_*\|_{\left(\sum_{x \in \mathcal{X}}\lambda_x x x^\top\right)}^2
\end{align*}
with $\overline{S}_{x_*} = \{\theta \in \mathcal{E} \; \textrm{s.t.}\; \exists x' \neq x_*, \theta^\top (x' - x_*) > 0\}$ and with here $\mathcal{E} = \{\theta \in \mathbb{R}^d : \|\theta\|_2^2\leq R^2\}$. \\\\
We can now start the proof by writing $T^{*-1}(\theta)$ as
\begin{align*}
    T^{*-1}(\theta) = \max_{\lambda\in \Delta_{\mathcal{X}}}\inf_{x' \neq x_*}\inf_{ \theta \in \mathcal{E}, \theta^\top (x' - x_*) > 0} \|\theta - \theta_*\|_{\left(\sum_{x \in \mathcal{X}}\lambda_x x x^\top\right)}^2\;.
\end{align*}
Then, instead of 
\begin{align*}
    \inf_{ \theta \in \mathcal{E}, \theta^\top (x' - x_*) > 0} \|\theta - \theta_*\|_{\left(\sum_{x \in \mathcal{X}}\lambda_x x x^\top\right)}^2
\end{align*}
we use $y = x' - x_*$ to write
\begin{align*}
    \inf_{ \theta \in \mathbb{R}^d, \theta^\top y \geq 0, \|\theta\|_2^2\leq R^2} \frac{1}{2}\|\theta - \theta_*\|_{A(\lambda)}^2 \;.
\end{align*}
We introduce the Lagrangian of this convex program
\begin{align*}
    L(\theta, \gamma, \nu) =  \frac{1}{2}\|\theta - \theta_*\|_{A(\lambda)}^2 - \nu(\theta^\top y) + \frac{\gamma}{2} (\|\theta\|_2^2 - R^2) \;.
\end{align*}
and solve 
\begin{align*}
    \inf_{ \theta \in \mathbb{R}^d} L(\theta, \gamma, \nu) &= \inf_{ \theta \in \mathbb{R}^d}\frac{1}{2}\|\theta - \theta_*\|_{A(\lambda)}^2 - \nu(\theta^\top y) + \frac{\gamma}{2}( \|\theta\|_2^2 - R^2)
\end{align*}
$\theta \mapsto L(\theta, \gamma, \nu)$ is differentiable and convex so we compute the gradient 
\begin{align*}
    \nabla_\theta L(\theta, \gamma, \nu) = (A(\lambda) + \gamma I) \theta - A(\lambda) \theta_* - \nu y
\end{align*}
and set it to zero to get
\begin{align*}
    \widehat{\theta} = \arg \min_{ \theta \in \mathbb{R}^d} L(\theta, \gamma, \nu) = (A(\lambda) + \gamma I)^{-1} (A(\lambda) \theta_* + \nu y) = \theta_* + (A(\lambda) + \gamma I)^{-1} (\nu y - \gamma \theta_*) 
\end{align*}
The cross term of both norms has absolute value $\gamma \nu \theta_* (A(\lambda) + \gamma I)^{-1}A(\lambda) y$, and they cancel. So we get
\begin{align*}
    L(\widehat{\theta}, \gamma, \nu)
    &= \frac{1}{2}\|(A(\lambda) + \gamma I)^{-1} (\nu y - \gamma \theta_*) \|_{A(\lambda)}^2 - \nu y^\top\left((A(\lambda) + \gamma I)^{-1} (A(\lambda) \theta_* + \nu y)\right)  \\
    &\;\;\;\;+ \frac{\gamma}{2} (\|(A(\lambda) + \gamma I)^{-1} (A(\lambda) \theta_* + \nu y)\|_2^2 - R^2)\\
    &= \frac{\gamma^2}{2}\|(A(\lambda) + \gamma I)^{-1} \theta_* \|_{A(\lambda)}^2 + \frac{\gamma}{2}\|(A(\lambda) + \gamma I)^{-1} A(\lambda) \theta_* \|^2 - \frac{\gamma}{2} R^2 - \nu y^\top (A(\lambda) + \gamma I)^{-1}A(\lambda) \theta_* \\
    &\;\;\;\;+ \frac{\nu^2}{2}\|(A(\lambda) + \gamma I)^{-1} y\|_{A(\lambda)}^2 - \nu^2 y^\top (A(\lambda) + \gamma I)^{-1} y + \frac{\nu^2\gamma}{2} \|(A(\lambda) + \gamma I)^{-1} y \|^2\\
    &= \frac{\gamma}{2} \left(\|\theta_*\|_{(A(\lambda) + \gamma I)^{-1}A(\lambda)}^2 - R^2\right) - \nu y^\top (A(\lambda) + \gamma I)^{-1}A(\lambda) \theta_* - \frac{\nu^2}{2}\|y\|^2_{(A(\lambda) + \gamma I)^{-1}}
\end{align*}
so 
\begin{align*}
    \sup_{\nu \geq 0}L(\widehat{\theta}, \gamma, \nu) = \frac{\gamma}{2} \left(\|\theta_*\|_{(A(\lambda) + \gamma I)^{-1}A(\lambda)}^2 - R^2\right) + \frac{(\max\{y^\top (A(\lambda) + \gamma I)^{-1}A(\lambda) \theta_*, 0\})^2}{2\|y\|^2_{(A(\lambda) + \gamma I)^{-1}}} 
\end{align*}
Conclusion:
\begin{align*}
    \inf_{ \theta \in \mathbb{R}^d, \theta^\top y \geq 0, \|\theta\|_2^2\leq R^2} \frac{1}{2}\|\theta - \theta_*\|_{A(\lambda)}^2 
    &= \sup_{\gamma \geq 0} \left\{\frac{(\max\{y^\top (A(\lambda) + \gamma I)^{-1}A(\lambda) \theta_*, 0\})^2}{2\|y\|^2_{(A(\lambda) + \gamma I)^{-1}}} + \frac{\gamma}{2} \left(\|\theta_*\|_{(A(\lambda) + \gamma I)^{-1}A(\lambda)}^2 - R^2\right)\right\}\\
\end{align*}
Then for any $\theta_* \in \mathcal{E}$ we have 
\begin{align*}
    \liminf_{\delta \rightarrow 0} \frac{\mathbb{E}_{\theta_*}[\tau_\delta]}{\log(1/\delta)} &\geq T^*(\theta_*) \\
    &= \frac{1}{\max_{\lambda\in \Delta_{\mathcal{X}}}\inf_{x' \neq x_*}\sup_{\gamma \geq 0} \left\{\frac{(\max\{(x'-x_*)^\top (A(\lambda) + \gamma I)^{-1}A(\lambda) \theta_*, 0\})^2}{2\|x'-x_*\|^2_{(A(\lambda) + \gamma I)^{-1}}} + \frac{\gamma}{2} \left(\|\theta_*\|_{(A(\lambda) + \gamma I)^{-1}A(\lambda)}^2 - R^2\right)\right\}}\\
    &=\inf_{\lambda\in \Delta_{\mathcal{X}}}\sup_{x' \neq x_*}\inf_{\gamma \geq 0}\frac{1}{F(\lambda, x', \gamma, \theta_*)}
\end{align*}
With
\begin{align*}
    F(\lambda, x', \gamma, \theta_*) &:= \frac{\max\{(x'-x_*)^\top (A(\lambda) + \gamma I)^{-1}A(\lambda) \theta_*, 0\}^2}{2\|x'-x_*\|^2_{(A(\lambda) + \gamma I)^{-1}}} + \frac{\gamma}{2} \left(\|\theta_*\|_{(A(\lambda) + \gamma I)^{-1}A(\lambda)}^2 - R^2\right)
\end{align*}
\end{proof}
\noindent Note that this result and its proof can be written in the case where $\phi$ is any feature map without any changes.

\section{Experiments details}

We briefly provide some additional details on the experiments. We used Python 3 and parallelized the simulations on a 2.9 GHz Intel Core i7. We computed the designs in each of the three experiments using mirror descent. We repeated the G-optimal design experiment 16 times, the kernels experiment 40 times, and the IPS vs. RIPS experiment 16 times. The G-optimal design experiment and the IPS vs. RIPS experiment used noise $\eta \sim N(0,1)$ while in the kernels experiment used noise $\eta \sim N(0,0.05)$. The confidence bounds in our plots are based on standard errors. 

\end{document}